\documentclass{article} 
\usepackage{newtxtext} 
\usepackage{natbib}
\setcitestyle{authoryear,round,citesep={;},aysep={,},yysep={;}}
\usepackage[hyphens]{url}
\usepackage{breakcites}
\usepackage{xcolor}

\usepackage[pdfpagelabels]{hyperref}
\hypersetup{
     colorlinks = true,
     linkcolor = black,
     anchorcolor = black,
     citecolor = black,
     filecolor = darkgray,
     urlcolor = darkgray
     }
\usepackage{amsmath, amssymb, amsthm}
\usepackage{fancyhdr}

\makeatletter

\def\addcontentsline#1#2#3{}
\def\maketitle{\par
\begingroup
   \def\thefootnote{\fnsymbol{footnote}}
   \def\@makefnmark{\hbox to 0pt{$^{\@thefnmark}$\hss}} 
   \long\def\@makefntext##1{\parindent 1em\noindent
                            \hbox to1.8em{\hss $\m@th ^{\@thefnmark}$}##1}
   \@maketitle \@thanks
\endgroup
\setcounter{footnote}{0}
\let\maketitle\relax \let\@maketitle\relax
\gdef\@thanks{}\gdef\@author{}\gdef\@title{}\let\thanks\relax}

\renewenvironment{abstract}{\vskip.075in\centerline{\large\sc
Abstract}\vspace{0.5ex}\begin{quote}}{\par\end{quote}\vskip 1ex}

\def\section{\@startsection {section}{1}{\z@}{-2.0ex plus
    -0.5ex minus -.2ex}{1.5ex plus 0.3ex
minus0.2ex}{\large\sc\raggedright}}

\def\subsection{\@startsection{subsection}{2}{\z@}{-1.8ex plus
-0.5ex minus -.2ex}{0.8ex plus .2ex}{\normalsize\sc\raggedright}}
\def\subsubsection{\@startsection{subsubsection}{3}{\z@}{-1.5ex
plus      -0.5ex minus -.2ex}{0.5ex plus
.2ex}{\normalsize\sc\raggedright}}
\def\paragraph{\@startsection{paragraph}{4}{\z@}{1.5ex plus
0.5ex minus .2ex}{-1em}{\normalsize\bf}}
\def\subparagraph{\@startsection{subparagraph}{5}{\z@}{1.5ex plus
  0.5ex minus .2ex}{-1em}{\normalsize\sc}}

\skip\footins 9pt plus 4pt minus 2pt
\def\footnoterule{\kern-3pt \hrule width 12pc \kern 2.6pt }
\leftmargini\leftmargin \leftmarginii 2em
\def\@listi{\leftmargin\leftmargini}
\def\@listii{\leftmargin\leftmarginii
   \labelwidth\leftmarginii\advance\labelwidth-\labelsep
   \topsep 2pt plus 1pt minus 0.5pt
   \parsep 1pt plus 0.5pt minus 0.5pt
   \itemsep \parsep}
\def\@listiii{\leftmargin\leftmarginiii
    \labelwidth\leftmarginiii\advance\labelwidth-\labelsep
    \topsep 1pt plus 0.5pt minus 0.5pt
    \parsep \z@ \partopsep 0.5pt plus 0pt minus 0.5pt
    \itemsep \topsep}
\def\@listiv{\leftmargin\leftmarginiv
     \labelwidth\leftmarginiv\advance\labelwidth-\labelsep}
\def\@listv{\leftmargin\leftmarginv
     \labelwidth\leftmarginv\advance\labelwidth-\labelsep}
\def\@listvi{\leftmargin\leftmarginvi
     \labelwidth\leftmarginvi\advance\labelwidth-\labelsep}

\belowdisplayskip \abovedisplayskip
\def\normalsize{\@setsize\normalsize{11pt}\xpt\@xpt}
\def\small{\@setsize\small{10pt}\ixpt\@ixpt}
\def\footnotesize{\@setsize\footnotesize{10pt}\ixpt\@ixpt}
\def\scriptsize{\@setsize\scriptsize{8pt}\viipt\@viipt}
\def\tiny{\@setsize\tiny{7pt}\vipt\@vipt}
\def\large{\@setsize\large{14pt}\xiipt\@xiipt}
\def\Large{\@setsize\Large{16pt}\xivpt\@xivpt}
\def\LARGE{\@setsize\LARGE{20pt}\xviipt\@xviipt}
\def\huge{\@setsize\huge{23pt}\xxpt\@xxpt}
\def\Huge{\@setsize\Huge{28pt}\xxvpt\@xxvpt}

\makeatother

\theoremstyle{plain}
\newtheorem{theorem}{Theorem}
\newtheorem{lemma}[theorem]{Lemma}
\newtheorem*{klemma}{Key Lemma}
\newtheorem{remark}[theorem]{Remark}

\newtheorem*{observation*}{Observation}
\newtheorem{observation}{Observation}
\newtheorem{example}{Example}

\newcommand{\ldbl}{\{\!\!\{}
\newcommand{\rdbl}{\}\!\!\}}
\newcommand{\Bldbl}{\bigl\{\!\!\bigr\{}
\newcommand{\Brdbl}{\bigl\}\!\!\bigr\}}

\newcommand{\Rb}{\mathbb{R}}

\newcommand{\Nb}{\mathbb{N}}

\newcommand{\countinglogic}[1]{\mathsf{C}_{#1}}
\newcommand{\ign}[1]{\mathsf{#1\text{-}IGN}}
\newcommand{\igns}[1]{\mathsf{#1\text{-}IGN}\text{s}}
\newcommand{\mlp}{\mathsf{MLP}}
\newcommand{\wl}[1]{\mathsf{#1\text{-}WL}}
\newcommand{\fwl}[1]{\mathsf{#1\text{-}FWL}}
\newcommand{\MPNN}{\mathsf{MPNN}}
\newcommand{\MPNNs}{\mathsf{MPNN}\text{s}}
\newcommand{\GNN}{\mathsf{GNN}}
\newcommand{\GNNs}{\mathsf{GNN}\text{s}}

\newcommand{\PPGNs}[1]{\mathsf{#1\text{-}PPGN}\text{s}}

\newcommand{\gns}[1]{\mathsf{#1\text{-}GNN}\text{s}}
\newcommand{\free}[1]{\mathsf{free}(#1)}

\title{The expressive power of $k$th-order invariant graph networks}

\author{Floris Geerts\\
University of Antwerp\\
\texttt{floris.geerts@uantwerp.be}}

\begin{document}

\maketitle

\begin{abstract}
The expressive power of graph neural network formalisms is commonly measured by their ability to distinguish graphs. For many formalisms, the $k$-dimensional Weisfeiler-Leman ($\wl{k}$) graph isomorphism test is used as a yardstick. In this paper we consider the expressive power of $k$th-order invariant (linear) graph  networks ($\igns{k}$). It is known that $\igns{k}$ are expressive enough to simulate $\wl{k}$. This means that for any two graphs that can be distinguished by $\wl{k}$, one can find a $k$-IGN which also distinguishes those graphs. The question remains whether $\igns{k}$ can distinguish more graphs than $\wl{k}$. This was recently shown to be false for $k=2$. Here, we generalise this result to arbitrary $k$. In other words, we show that $\igns{k}$ are bounded in expressive power by $\wl{k}$. This implies that $\igns{k}$ and $\wl{k}$ are equally powerful in distinguishing graphs.
\end{abstract}

\section{Introduction}\label{sec:introduction}
Graph neural networks ($\GNNs$) have become a standard means to analyse graph data. One of the most widely adopted $\GNN$ formalisms are the so-called message-passing neural networks ($\MPNNs$)~\citep{scarselli2008graph,GilmerSRVD17}. In $\MPNNs$, features of vertices are iteratively updated based on the features of neighbouring vertices, and the current feature of the vertex itself. In their simplest form, when only the features of vertices are taken into account, the capability of $\MPNNs$ to distinguish vertices and graphs is rather limited. Indeed, \citet{xhlj19} and \citet{grohewl} show that the expressive power of $\MPNNs$ is bounded by the 1-dimensional (Folklore) Weisfeiler-Leman ($\fwl{1}$) graph isomorphism test \citep{CaiFI92}, or equivalently, the 2-dimensional Weisfeiler-Leman ($\wl{2}$) test \citep{grohe_otto_2015,grohe_2017}\footnote{In works related to Weisfeiler-Leman one has to carefully consider whether or not the Folklore $\mathsf{WL}$ test is used. That is, in some papers, $\wl{1}$ refers to $\fwl{1}$. For general $k$, $\fwl{k}$ is equivalent to $\wl{(k+1)}$~\citep{grohe_otto_2015}.}.  That is, when two graphs cannot be distinguished by $\wl{2}$, then neither can they be distinguished by any $\MPNN$.  The expressive power of $\wl{2}$ is well-understood. For example, when two graphs cannot be distinguished by $\wl{2}$ then they can also not be distinguished by sentences in the two-variable fragment, $\countinglogic{2}$, of first-order logic with counting. More  relevant in the context of $\GNNs$ is the complete characterisation of $\wl{2}$ in terms of invariant graph properties~\citep{Furer17,ARVIND202042}. For example, $\wl{2}$ is unable to detect cycles of length greater than four or triangles in graphs. We also like to point out connections between $\wl{2}$ and homomorphism profiles. More specifically, two graphs are indistinguishable by $\wl{2}$ if and only if they have the same number of homomorphisms from graphs of treewidth at most one \citep{DellGR18}. Finally, one can rephrase indistinguishability by $\wl{2}$ in terms of agreement of functions defined in terms of linear algebra operators \citep{Geerts19}. \looseness=-1

The limited expressive power of $\MPNNs$ is primarily due to the fact that vertices are anonymous, i.e., two vertices with the same feature are regarded as equivalent, and that only neighbouring vertices are considered. When, for example, $\MPNNs$ are degree-aware, meaning that they can distinguish vertices based on both their features and degrees, $\MPNNs$ get a slight jump start when compared to $\wl{2}$ and can potentially distinguish graphs in one iteration earlier than $\wl{2}$~\citep{geerts2020let}. Notable examples of degree-aware $\MPNNs$ are the graph convolutional networks by \citet{kipf-loose}. More powerful variants of $\MPNNs$ can be obtained by incorporating port numbering, which allows to treat features from different neighbours differently~\citep{sato2019approximation}, assigning random initial features~\citep{sato2020random}, and having static vertex identifiers \citep{Loukas2020What}. We refer \citet{Sato2020ASO} for a more detailed overview of these and other variations of $\MPNNs$.

Instead of considering $\wl{2}$ or variations of standard $\MPNNs$, this paper concerns $\GNNs$ inspired by the $k$-dimensional Weisfeiler-Leman ($\wl{k}$) graph isomorphism test, for $k\geq 2$. These tests iteratively update features of $k$-tuples of vertices, based on the features of neighbouring $k$-tuples of vertices. It is known that the expressive power of $\wl{k}$ grows with increasing $k$~\citep{CaiFI92}. As such, they provide a promising basis for the development of more expressive $\GNNs$. Of particular interest is the ability of $\wl{k}$, for $k\geq 2$, to distinguish graphs based on the presence or absence of specific graph patterns, such as cycles and cliques. For example, $\wl{3}$ can distinguish graphs based on their number of cycles of length up to $7$ and triangles~\citep{Furer17,Geerts19,ARVIND202042}. Furthermore, graphs that are indistinguishable by $\wl{k}$ satisfy the same sentences in $\countinglogic{k}$, the $k$-variable fragment of first-order logic with counting \citep{CaiFI92}, and this in turn is equivalent to the two graphs having the same number of homomorphisms from graphs of treewidth at most $k-1$ \citep{DellGR18}. The latter correspondence has led \citet{nt2020graph} to define $\GNNs$ based on graph homomorphism convolutions.
We refer to \citet{Grohe20} for other interesting interpretations of $\wl{k}$ and relationships to embeddings of graph, and more generally, structured data.
\looseness=-1

Given the promise of an increase in expressive power, \citet{grohewl} propose $\gns{k}$ based a set-variant of $\wl{k}$. We will not consider this set-variant of $\wl{k}$ in this paper and only mention that $\gns{k}$ match the set-variant of $\wl{k}$ in expressive power. More relevant to this paper is the work by \citet{DBLP:conf/nips/MaronBSL19} in which it is shown that the  class of $k$th-order invariant graph networks ($\igns{k}$) is as powerful as $\wl{k}$ in expressive power, for each $k\geq 2$. In other words, when two graphs can be distinguished by $\wl{k}$, then there exists a $\ign{k}$ which also distinguishes those graphs. Invariant graph networks ($\ign{k}$) are built-up from equivariant layers defined over $k$th-order tensors \citep{kondor2018covariant,maron2018invariant}. By contrast to $\wl{k}$, $\igns{k}$ update features of $k$-tuples of vertices based on the features of \textit{all} $k$-tuples, i.e., not only those that are neighbours as in $\wl{k}$. As a consequence, it is not immediately clear that $\igns{k}$ are bounded by $\wl{k}$ in expressive power.
We remark, however, that in a $\ign{k}$, not all (features of) $k$-tuples are treated the same due to the equivariance of its layers. More precisely, given a $k$-tuple $\bar v$ of vertices,  the space of all $k$-tuples of vertices is partitioned according to which equality and inequality conditions are satisfied together with $\bar v$. Then, during the feature update process of $\bar v$, two $k$-tuples of vertices with the same feature may be treated differently by a $\ign{k}$ if the two $k$-tuples belong to different parts of the partition relative to $\bar v$.
\looseness=-1

\citet{openprob} raise the natural question whether, despite that $\igns{k}$ use more information than $\wl{k}$, the expressive power of $\igns{k}$ is still limited to that of $\wl{k}$. In other words, can there be graphs that can be distinguished by a $\ign{k}$ which cannot be distinguished by $\wl{k}$. This question was recently answered by \citet{chen2020graph} for $k=2$. More precisely, they show that, for undirected graphs, the expressive power of $\igns{2}$ is indeed bounded by $\wl{2}$. Furthermore, there is a one-to-one correspondence between the layers in a $\ign{2}$ and iterations in $\wl{2}$. That is, when two graphs cannot by distinguished by $\wl{2}$ in $t$ iterations, then neither can they be distinguished by a $\ign{2}$ using $t$ equivariant layers.

In this paper, we generalise this result to arbitrary $k$. More precisely, we show that the expressive power of $\igns{k}$ is indeed bounded by $\wl{k}$. What is interesting to note is that the one-to-one correspondence between iterations of $\wl{k}$ and layers in $\igns{k}$ needs to be revisited. As it turns out, for general $k$, each layer of a $\ign{k}$ can be seen to correspond to $k-1$ iterations by $\wl{k}$. We remark that when $k=2$, the one-to-one correspondence from \citet{chen2020graph} is recovered. This implies that, in principle, a $\ign{k}$ can distinguish graphs a factor of $k-1$ faster  compared to $\wl{k}$. Of course, this comes at a cost of a more intensive feature update process involving all $k$-tuples of vertices. \citet{chen2020graph} establish their result for $k=2$ in a pure combinatorial way and by means of a case analysis, which is feasible for a fixed $k$. For general $k$, we
borrow ideas from \citet{chen2020graph}  but additionally rely on the known connection between $\wl{k}$ and the logic $\countinglogic{k}$ mentioned earlier. We remark that connections with logic, $\MPNNs$ and $\wl{2}$ have been used before to assess the logical expressiveness of $\MPNNs$ \citep{barcelo2019logical}.

We also remark that $\igns{k}$ incur a large cost in memory and computation. Alternatives to $\igns{k}$ are put forward based on the folklore $k$-dimensional Weisfeiler-Leman ($\fwl{k}$) test, which is known to 
be more efficient to implement. For example, \citet{DBLP:conf/nips/MaronBSL19} propose provably powerful graph networks ($\PPGNs{k}$) that are able to simulate $\fwl{k}$ (and thus $\wl{(k+1)}$) by using $k$th-order tensors only but in which the layers  are allowed to use tensor multiplication.  For $\fwl{2}$, a single matrix multiplication suffices. The impact of matrix multiplication in layers has been further investigated in \citet{geerts2020walk}. In that work, inspired by the work of \citet{LichterPS19}, walk $\MPNNs$ are proposed as a general formalism for $\PPGNs{2}$. It is readily verified that walk $\MPNNs$ are bounded in expressive power by $\fwl{2}$, and since $\PPGNs{2}$ can be seen as instances of walk $\MPNNs$, they are bounded in expressive power by $\fwl{2}$ as well \citep{geerts2020walk}. This has been generalised by \citet{azizian2020characterizing} who show that $\PPGNs{k}$ are bounded by $\fwl{k}$, for arbitrary $k$. We also note that allowing more than one matrix multiplication in  $\PPGNs{2}$ does not increase their expressive power. Instead, multiple matrix multiplications may result in that $\PPGNs{2}$ can distinguish graphs faster than $\fwl{2}$ \citep{geerts2020walk}.   In this paper, we only consider $\igns{k}$ and $\wl{k}$.
\looseness=-1
\paragraph{Structure of the paper.}
We start by describing $\wl{k}$, $\countinglogic{k}$ and $\igns{k}$ in Section~\ref{sec:background}. 
Then, in Section~\ref{sec:background} we prove that $\igns{k}$ are bounded by $\wl{k}$ in expressive power.
We conclude in Section~\ref{sec:conclude}.

\section{Background}\label{sec:background}
We first describe $\wl{k}$ and its connections to logic, followed by the definition of $\igns{k}$.
We use $\{\,\}$ to denote sets and $\ldbl\, \rdbl$ to denote multisets. The sets of natural and real numbers are denoted by $\Nb$ and $\Rb$, respectively. For $n\in\Nb$
with $n>0$, we define $[n]:=\{1,\ldots,n\}$. A (directed) graph $G=(V(G),E(G))$ consists of a vertex set $V(G)$ and edge set $E(G)\subseteq V^2$. A (vertex-)coloured graph $G=(V(G),E(G),\chi_G)$ is a graph in which every 
vertex $v\in V(G)$ is assigned a colour $\chi_G(v)$ in some set $\mathcal{C}$ of colours. In the following, when we refer to graphs we always mean coloured graphs. Without loss of generality we assume that $V(G)=[n]$ for some $n\in\Nb$. Furthermore, if $\mathbf{A}\in \Rb^{n^k\times p}$ is a $k$th-order tensor, then we denote by 
$\mathbf{A}_{\bar v, s}\in \Rb$ with $\bar v\in[n]^k$ and $s\in [p]$ the value of $\mathbf{A}$ in entry $(\bar v,s)$, and $\mathbf{A}_{\bar v,\bullet}\in\Rb^p$ denotes the vector $\bigl(\mathbf{A}_{\bar v,s}\bigm| s\in[p]\bigr)$ in $\Rb^p$.

\subsection{Weisfeiler-Leman}
The $k$-dimensional Weisefeiler-Leman ($\wl{k}$) graph isomorphism test iteratively produces colourings of $k$-tuples of vertices, starting from a given graph $G=(V(G),E(G),\chi_G)$. We follow here the presentation as given in~\citet{grohewl}.
Given $G=(V(G),E(G),\chi_G)$, we denote by $\chi_{G,k}^{(t)}:[n]^k\to\mathcal{C}$ the colouring of $k$-tuples generated by $\wl{k}$ after $t$ rounds. For $t=0$, $\chi_{G,k}^{(0)}:[n]^k\to\mathcal{C}$ is a colouring in which each $k$-tuple $\bar v\in [n]^k$ is coloured with the isomorphism type of its induced subgraph. More specifically,
$\chi_{G,k}^{(0)}(v_1,\ldots,v_k)=\chi_{G,k}^{(t)}(v_1',\ldots,v_k')$ if and only if for all $i\in[k]$ we have that $\chi_G(v_i)=\chi_G(v_i')$ and for all $i,j\in[k]$, it holds that $v_i=v_j$ if and only if $v_i'=v_j'$
and $(v_i,v_j)\in E(G)$ if and only if $(v_i',v_j')\in E(G)$. Then, for $t>0$, we define the colouring
$\chi_{G,k}^{(t)}:[n]^k\to\mathcal{C}$ as
$$
\chi_{G,k}^{(t)}(\bar v):=\textsc{Hash}\Bigl(\chi_{G,k}^{(t-1)}(\bar v),\bigl(C_1^{(t)}(\bar v),\ldots,C_k^{(t)}(\bar v)\bigr)\Bigr),
$$
in which for $i\in[k]$,
$$
C_i^{(t)}(\bar v):=\textsc{Hash}\Bigl(\Bldbl \chi_{G,k}^{(t-1)}(\bar v[v_i/v']) \bigm| v'\in [n]\Brdbl\Bigr),
$$
where $\bar v[v_i/v']:=(v_1,\ldots,v_{i-1},v',v_{i+1},\ldots,v_k)$ and $\textsc{Hash}(\cdot)$ is a hash function that maps it input in an injective manner to a colour in $\mathcal{C}$.

Let $\chi_1,\chi_2:[n]^k\to \mathcal{C}$ be colourings of $k$-tuples of vertices in $G$. We say that $\chi_1$ refines $\chi_2$, denoted by $\chi_1\preceq \chi_2$, if for all $\bar v,\bar v'\in [n]^k$ we have $\chi_1(\bar v)=\chi_1(\bar v')\Rightarrow\chi_2(\bar v)=\chi_2(\bar v')$.
When $\chi_1\preceq\chi_2$ and $\chi_2\preceq\chi_1$ hold, we say that $\chi_1$ and $\chi_2$ are equivalent and we denote this by $\chi_1\equiv\chi_2$.

We note that, by definition,
$\chi_{G,k}^{(t)}\preceq \chi_{G,k}^{(t-1)}$ for all $t\geq 1$. We define $\chi_{G,k}$ as $\chi_{G,k}^{(t)}$ for which
$\chi_{G,k}^{(t)}\equiv \chi_{G,k}^{(t+1)}$ holds.  It is known that this ``stable'' colouring is obtained in a most $n^{\mathcal{O}(k)}$ rounds.
For two graphs $G=(V(G),E(G),\chi_G)$ and $H=(V(H),E(H),\chi_H)$, one says that $\wl{k}$ distinguishes $G$ and $H$ in round $t$ if 
$$
\Bldbl\chi_{G,k}^{(t)}(\bar v)\bigm| \bar v\in (V(G))^k\Brdbl\neq
\Bldbl\chi_{H,k}^{(t)}(\bar w)\bigm| \bar w\in (V(H))^k\Brdbl.
$$
We write $G\equiv_{\wl{k}}^t H$ if $\wl{k}$ does not distinguish $G$ and $H$ in round $t$. When
$G\equiv_{\wl{k}}^t H$ for all $t\geq 0$, we write $G\equiv_{\wl{k}} H$ and say that $G$ and $H$ cannot be distinguished by $\wl{k}$.
	
\subsection{Counting logics}
The $k$-dimensional Weisfeiler-Leman graph isomorphism test is closely tied to the $k$-variable fragment of first-order logic with counting, denoted by $\countinglogic{k}$, on graphs. This logic is defined over a finite set of $k$ variables, $x_1,\ldots,x_k$, and a formula $\varphi$ in $\countinglogic{k}$ is  formed according to the following grammar:
$$\varphi ::= x_i=x_j \ \mid \ \mathsf{Col}_c(x_i) \ \mid \ \mathsf{Edge}(x_i,x_j)
\ \mid \ \neg\varphi \!\ \mid \ \varphi_1\land\varphi_2 \ \mid \ \exists^{\geq r} x_i\, \varphi,$$
for $i,j\in[k]$, $c\in\mathcal{C}$, $r\in\Nb$ with $r>0$. The first three cases in the grammar correspond to so-called atomic formulas. For a formula $\varphi$, we define its free variables $\free{\varphi}$ in an inductive way, i.e., $\free{x_i=x_j}:=\{x_i,x_j\}$. 
$\free{\mathsf{Col}_c(x_i)}:=\{x_i\}$ $\free{\mathsf{Edge}(x_i,x_j)}:=\{x_i,x_j\}$, 
$\free{\neg\varphi}:=\free{\varphi}$, $\free{\varphi_1\land \varphi_2}:=\free{\varphi_1}\cup\free{\varphi_2}$, and $\free{\exists^{\geq r}x_i\,\varphi}:=\free{\varphi}\setminus\{x_i\}$. We write $\varphi(x_1,\ldots,x_k)$ to indicate that all free variables of $\varphi$ are among $x_1,\ldots,x_k$. A sentence is formula without free variables. We further need the quantifier rank of a formula $\varphi$, denoted by $\mathsf{qr}(\varphi)$. It is defined as follows: $\mathsf{qr}(\varphi):=0$ if $\varphi$ is atomic, 
$\mathsf{qr}(\neg\varphi):=\mathsf{qr}(\varphi)$, $\mathsf{qr}(\varphi_1\land\varphi_2):=\mathsf{max}\{\mathsf{qr}(\varphi),\mathsf{qr}(\varphi_2)\}$, and 
$\mathsf{qr}(\exists^{\geq r}x_i\,\varphi):=\mathsf{qr}(\varphi)+1$.

Let $G=(V(G),E(G),\chi_G)$ be a graph and let $\varphi(x_1,\ldots,x_k)$ be a formula in $\countinglogic{k}$. Consider an assignment $\alpha$ from the variables  $\{x_1,\ldots,x_k\}$ to vertices in $V(G)$. We denote by
$\alpha(x_i/v)$ for $v\in V(G)$ the assignment which is equal to $\alpha$ except that $\alpha(x_i):=v$.
We define the satisfaction of a formula by a graph, relative to an assignment $\alpha$, denoted by $G\models \varphi[\alpha]$, in an inductive manner. That is, $G\models (x_i=x_j)[\alpha]$ if and only if $\alpha(x_i)=\alpha(x_j)$,
$G\models \mathsf{Col}_c(x_i)[\alpha]$ if and only if $\chi_H(\alpha(x_i))=c$, 
$G\models \mathsf{Edge}(x_i,x_j)[\alpha]$ if and only if $(\alpha(x_i),\alpha(x_j))\in E$, 
$G\models \neg\varphi[\alpha]$ if and only if not $G\models\varphi[\alpha]$, 
$G\models (\varphi_1\land \varphi_2)[\alpha]$ if and only if $G\models\varphi_1[\alpha]$ and $G\models\varphi_2[\alpha]$, and finally, $G\models \exists^{\geq r} x_i\, \varphi[\alpha]$ if and only if there are at least $r$ distinct vertices $v_1,\ldots,v_r$ in $V(G)$ such that 
$G\models\varphi[\alpha(x_i/v_j)]$ holds for all $j\in[r]$.
\looseness=-1

When $G$ and $H$ satisfy the same sentences in $\countinglogic{k}$ of quantifier rank at most $t$, 
we denote this by $G\equiv_{\countinglogic{k}}^t H$. If $G\equiv_{\countinglogic{k}}^t H$ holds for all $t\geq 0$,
then we write $G\equiv_{\countinglogic{k}} H$ and say that $G$ and $H$ are indistinguishable by $\countinglogic{k}$.  The connection to $\wl{k}$ is as follows.
\begin{theorem}[\citep{CaiFI92}] \label{thm:cfi}
Let $G$ and $H$ be two graphs. Then, $G\equiv_{\wl{k}}^t H$ if and only if $G\equiv_{\countinglogic{k}}^t H$.
As a consequence, $G\equiv_{\wl{k}} H$ if and only if $G\equiv_{\countinglogic{k}} H$.\qed
\end{theorem}
Of particular interest is that the proof of this theorem shows that, for $c\in\mathcal{C}$, there 
exists a formula $\psi_c^{(t)}(x_1,\ldots,x_k)$ in $\countinglogic{k}$ of quantifier rank at most $t$ such $\chi_{G,k}^{(t)}(v_1,\ldots,v_k)=c$ if and only if $G\models \psi_c^{(t)}[\alpha]$ with $\alpha$ defined as $x_i\mapsto v_i$.

Later in the paper we also use the shorthand notation $\exists^{\geq r}(x_1,\ldots,x_\ell)\,  \varphi$ to indicate that are at least $m$ distinct $\ell$-tuples satisfying $\varphi$. It is readily verified\footnote{I would like to acknowledge Jan Van den Bussche for pointing this out.} that if $\varphi$ is a formula in $\countinglogic{k}$ of quantifier rank $t$, then $\exists^{\geq r}(x_1,\ldots,x_\ell)\,  \varphi$ is equivalent to a formula in 
$\countinglogic{k}$ of quantifier rank at most $t+\ell$. Here, two formulas $\varphi$
and $\psi$ are equivalent if $G\models\varphi[\alpha]$ if and only if $G\models\psi[\alpha]$ for all assignments $\alpha$ and graphs $G$. As a consequence, quantifiers of the form $\exists^{\geq r}(x_1,\ldots,x_\ell)\,  \varphi$ for $\ell>1$ do not add expressive power to $\countinglogic{k}$.
 In what follows, for a formula $\varphi(x_1,\ldots,x_k)$ and assignment $\alpha$,
we write $\varphi[v_1,\ldots,v_k]$ instead of $\psi_c^{(t)}[\alpha]$ with $\alpha$ such that $x_i\mapsto v_i$. 

\subsection{Invariant graph neural networks}
Let $S_n$ denote the symmetric group over $[n]$, i.e., $S_n$ consists of all permutation $\pi$ of $[n]$.
Let $\pi\in S_n$ and $\mathbf{A}$ a tensor in $\Rb^{n^k\times p}$. We define
$\pi\star \mathbf{A}\in \Rb^{n^k\times p}$ such that 
$(\pi\star \mathbf{A})_{\pi(\bar v),\bullet}=\mathbf{A}_{\bar v,\bullet}$ for all $\bar v\in [n]^k$.
A $k$th-order equivariant linear layer is a mapping $L:\Rb^{n^k\times p}\to \Rb^{n^\ell\times q}$ such that
$L(\pi\star\mathbf{A})=\pi\star L(\mathbf{A})$ for all $\mathbf{A}\in \Rb^{n^k\times p}$. When $\ell=0$, and
thus $L(\pi\star\mathbf{A})=L(\mathbf{A})$ for all $\mathbf{A}\in \Rb^{n^k\times p}$, one refers to $L$ as an invariant layer.  An explicit description of equivariant linear layers was provided by \citet{maron2018invariant} and is based on the observation that such a layer is constant on equivalence
classes of $[n]^k$ defined by equality patterns. More specifically, let $\bar v$ and $\bar v'$ be $k$-tuples in $[n]^k$.
Then $\bar v$ and $\bar v'$ are said to have the same equality pattern, denoted by $\bar v\sim\bar v'$, if for all $i,j\in[k]$, $v_i=v_j$ if and only if 
$v_i'=v_j'$. We denote the set of equivalence classes in $[n]^k$ induced by $\sim$ by $[n]^k/_\sim$. Given this, an equivariant
layer $L:\Rb^{n^k\times p}\to \Rb^{n^k\times q}$ is of the form 
	\allowdisplaybreaks
\begin{align*}
	L(\mathbf{A})_{\bar v,a}&=\sum_{\mu\in [n]^{2k}/_\sim} L_\mu(\mathbf{A})_{\bar v,a} + \sum_{\substack{\tau\in [n]^k/_\sim\\\bar v\in\tau}} c_{\tau,a}, \text{with}\\
	L_\mu(\mathbf{A})_{\bar v,a}&=\sum_{\substack{\bar v'\in[n]^k\\(\bar v,\bar v')\in\mu}}
\bigl(\sum_{b\in [p]} c_{\mu,a,b}\mathbf{A}_{\bar v',b}\bigl)
\end{align*}
for $\bar v\in[n]^k$, $a\in[q]$ and $c_{\mu,a,b},c_{\tau,b}\in\Rb$.
An equality pattern $\mu\in[n]^{2k}/_\sim$ can be equivalently described by a partition $[2k]=I_1\uplus\cdots\uplus I_r$ with the interpretation that $\bar v\in \mu$ if and only if
$v_i=v_j$ whenever $i,j\in I_s$ for some $s\in[r]$, and $v_i\neq v_j$ whenever
$i\in I_s$ and $j\in I_{s'}$ for $s\neq s'$ and $s,s'\in[r]$. We will use this representation of
equality patterns later in the paper.

\citet{maron2018invariant} define a $k$th-order invariant (linear) graph network ($\ign{k}$) as a function $F:\Rb^{n^k\times s_0}\to \Rb^{s}$ that can be decomposed
as 
$$
F=M\circ I \circ \sigma\circ L^{(t)}\circ \sigma \circ L^{(t-1)}\circ\cdots \circ\sigma\circ L^{(1)},
$$
where for $i\in[t]$, each layer $L^{(i)}$ is an equivariant linear  layer from $\Rb^{n^k\times {s_{i-1}}}\to \Rb^{n^k\times s_i}$, $\sigma$ is a pointwise non-linear activation function such as the $\mathsf{ReLU}$ function, $I$ is a linear invariant layer from $\Rb^{n^k\times s_t}\to \Rb^{s_{t+1}}$, and $M$ is a multi layer perceptron ($\mlp$) from $\Rb^{s_{t+1}}$ to $\Rb^{s}$. 

We next use $\igns{k}$ $F$ to define an equivalence relation on graphs. To do so, we first turn a graph $G=(V(G),E(G),\chi_G)$ into a tensor $\mathbf{A}_G\in\Rb^{n^k\times s_0}$. More precisely, we  first consider the initial $\wl{k}$ colouring $\chi_{G,k}^{(0)}:[n]^k\to\mathcal{C}$ (recall that we identified $V(G)$ with $[n]$). Then, suppose that $\chi_{G,k}^{(0)}$ assigns $s_0$ distinct colours $c_1,\ldots,c_{s_0}$ to the $k$-tuples in $[n]^k$. We identify each colour $c_i$ with the $i$th basis vector $\mathbf{b}_i$ in $\Rb^{s_0}$ and define for $\bar v\in[n]^k$ and $s\in[s_0]$, $(\mathbf{A}_G)_{\bar v,s}:=1$ if $\chi_{G,k}^{(0)}(\bar v)=c_s$ and $(\mathbf{A}_G)_{\bar v,s}:=0$ otherwise.
 Given this, we say that two graphs $G$ and $H$ are indistinguishable by a $\ign{k}$ $F$, denoted by $G\equiv_F H$, if and only if $F(\mathbf{A}_G)=F(\mathbf{A}_H)$. We also consider another equivalence relation defined in terms the equivariant part of an $\ign{k}$ $F$. More precisely, for $t>0$, let $F^{(t)}:\Rb^{n^k\times s_0}\to \Rb^{n^k\times s_t}$ defined by $F^{(t)}:=\sigma\circ L^{(t)}\circ \cdots\sigma\circ L^{(1)}$. We let 
 $F^{(0)}$ be the identity mapping from $\Rb^{n^k\times s_0}\to\Rb^{n^k\times s_0}$.
 We then denote by $G \equiv_{F}^t H$ that 
$$\Bldbl F^{(t)}(\mathbf{A}_G)_{\bar v,\bullet}\Bigm| \bar v\in (V(G))^k\Brdbl=\Bldbl F^{(t)}(\mathbf{A}_H)_{\bar w,\bullet}\Bigm| \bar w\in (V(H))^k\Brdbl.$$
In other words, when viewing the tensors $F^{(t)}(\mathbf{A}_G)$ and 
$F^{(t)}(\mathbf{A}_H)$ in $\Rb^{n^k\times s_t}$ as colouring of $k$-tuples, i.e., $\bar v\in (V(G))^k$ is assigned the ``colour'' $F^{(t)}(\mathbf{A}_G)_{\bar v,\bullet}\in \Rb^{s_t}$ and similarly, $\bar w\in (V(H))^k$ is assigned the ``colour'' $F^{(t)}(\mathbf{A}_H)_{\bar w,\bullet}\in \Rb^{s_t}$, then
$G \equiv_{F}^t H$ just says these labelings are equivalent. In the remainder of the paper we establish correspondences between $\equiv_F^t$ and $\equiv_F$, and the equivalence relations $\equiv_{\wl{k}}^t$ and $\equiv_{\wl{k}}$.

\section{The expressive power of \texorpdfstring{$\igns{k}$}{k-IGNs}}\label{sec:express}
Let us start by recalling what is known about the relationship between the equivalence relations $\equiv_{\wl{k}}$ and $\equiv_F$. For every $k\geq 2$ and any two graphs $G$ and $H$, it is known that there exists a $\ign{k}$ $F$ such that 
$G \equiv_{F} H \Rightarrow G\equiv_{\wl{k}} H$ \citep{DBLP:conf/nips/MaronBSL19}. In other words, if $G$ and $H$ can be distinguished by $\wl{k}$, then the $\ign{k}$ $F$  distinguishes them as well. Hence, the class of $\igns{k}$ is powerful enough to match $\wl{k}$ in expressive power. The $\ign{k}$ $F$ used by \citet{DBLP:conf/nips/MaronBSL19} consists of $d$ equivariant layers, where $d$ is such that $\wl{k}$ reaches the stable colourings $\chi_{G,k}$ and $\chi_{H,k}$ of $G$ and $H$, respectively, in $d$ rounds. In fact, \citet{DBLP:conf/nips/MaronBSL19} show  that $G \equiv_{F}^t H \Rightarrow G\equiv_{\wl{k}}^t H$ holds as well, for $t\in[d]$, so the rounds of $\wl{k}$ and the layers of $F$ are in one-to-one correspondence.
It was posed as an open problem in \citet{openprob} whether or not $\igns{k}$ can  distinguish more graphs than $\wl{k}$. More specifically, the question is whether the implication $G\equiv_{\wl{k}} H\Rightarrow  G \equiv_{F} H$ also holds, and this for any $\ign{k}$ $F$. This question was recently answered for $k=2$. Indeed, \citet{chen2020graph} show that $G\equiv_{\wl{2}} H\Rightarrow  G \equiv_{F} H$ holds for any $\ign{2}$ $F$. As a consequence, $\wl{2}$ and $\igns{2}$ have equal distinguishing power. In proving $G\equiv_{\wl{2}} H\Rightarrow  G \equiv_{F} H$, \citet{chen2020graph} show first that, when $F$ consists of $d$ equivariant layers, then for each  $t\in[d]$ $G\equiv_{\wl{2}}^t H \Rightarrow G \equiv_{F}^t H$. By leveraging this, they then verify  $G\equiv_{\wl{2}}^t H\Rightarrow G \equiv_{F} H$. Since $G\equiv_{\wl{2}} H\Rightarrow G\equiv_{\wl{2}}^t H$ for all $t\geq 0$, the implication $G\equiv_{\wl{2}} H\Rightarrow  G \equiv_{F} H$ follows. We remark that \citet{chen2020graph} consider undirected graphs only.  We next generalise this result to arbitrary $k\geq 2$ and to directed graphs. In other words, our main result is:
\begin{theorem}\label{thm:main}
For any two graphs $G$ and $H$, $G\equiv_{\wl{k}} H\Rightarrow  G \equiv_{F} H$ for any $\ign{k}$ $F$.
\end{theorem}

This theorem will be proved, in analogy with the proof by  \citet{chen2020graph},  by using Lemmas~\ref{lem:mainpert} and~\ref{lem:fromeqtofull} below. The first lemma is the counterpart, for general $k$, of the
implication $G\equiv_{\wl{2}}^t H \Rightarrow G \equiv_{F}^t H$ by  \citet{chen2020graph}. We see, however, that the correspondence between rounds of $\wl{k}$ and layers in $\igns{k}$ is slightly more involved.
\begin{lemma}\label{lem:mainpert}
Let $F$ be a $\ign{k}$ consisting of $d$ equivariant layers and consider graphs $G$ and $H$.
Then for any $t\geq 0$,
\begin{equation} 
	G\equiv_{\wl{k}}^t H \Rightarrow G \equiv_{F}^{\lfloor\frac{t}{k-1}\rfloor} H.
	 \tag{$\dagger$}\label{eq:main}
\end{equation}
\end{lemma}
Note that when $k=2$, $\lfloor\frac{t}{k-1}\rfloor=t$ and hence the known implication for $k=2$ from  \citet{chen2020graph} is recovered. Since $F$
consists of $d$ layers, we limit $t$ to be in the range of $(d+1)(k-1)-1$ such that $\lfloor \frac{t}{k-1}\rfloor\leq d$. As part of the proof of Lemma~\ref{lem:mainpert} we show a stronger implication. More precisely, we show that if $G\equiv_{\wl{k}}^t H$ holds, then
\begin{equation*}
\chi_{G,k}^{(t)}(\bar v)= \chi_{H,k}^{(t)}(\bar w)\Rightarrow \bigl(F^{(\lfloor\frac{t}{k-1}\rfloor)}(\mathbf{A}_G)\bigr)_{\bar v,\bullet}=
\bigl(F^{(\lfloor\frac{t}{k-1}\rfloor)}(\mathbf{A}_H)\bigr)_{\bar w,\bullet},\label{eq:main2}
\end{equation*}
for any $\bar v\in (V(G))^k$ and $\bar w\in (V(H))^k$. We use this property in the next lemma.

\begin{lemma}\label{lem:fromeqtofull}
	Let $F$ be a $\ign{k}$ consisting of $d$ equivariant layers and consider graphs $G$ and $H$. Let $t=d(k-1)$ and assume that the following implication holds for $\bar v\in (V(G))^k$ and $\bar w\in (V(H))^k$, $\chi_{G,k}^{(t)}(\bar v)= \chi_{H,k}^{(t)}(\bar w)\Rightarrow \bigl(F^{(d)}(\mathbf{A}_G)\bigr)_{\bar v,\bullet}=
\bigl(F^{(d)}(\mathbf{A}_H)\bigr)_{\bar w,\bullet}$. Then
\begin{equation*}
	G\equiv_{\wl{k}}^t H\Rightarrow G\equiv_F H. \end{equation*} 
\end{lemma}

These two lemmas suffice to prove Theorem~\ref{thm:main}:
\begin{proof}
Indeed, suppose that $G\equiv_{\wl{k}} H$ holds.
By definition, this implies  $G\equiv_{\wl{k}}^t H$ for all $t\geq 0$.
In particular, this holds for $t=d(k-1)$.
As mentioned above, as part of proving
 Lemma~\ref{lem:mainpert} we obtain for
 $\bar v\in (V(G))^k$ and $\bar w\in (V(H))^k$, the implication $\chi_{G,k}^{(t)}(\bar v)= \chi_{H,k}^{(t)}(\bar w)\Rightarrow \bigl(F^{(d)}(\mathbf{A}_G)\bigr)_{\bar v,\bullet}=
 \bigl(F^{(d)}(\mathbf{A}_H)\bigr)_{\bar w,\bullet}$. Then, 
Lemma~\ref{lem:fromeqtofull} implies $G\equiv_F H$, as desired.
\end{proof}

Before showing the lemmas, we provide some intuiting behind the implication~(\ref{eq:main}) in  Lemma~\ref{lem:mainpert}. In a nutshell, it reflects that a single (equivariant) layer of a $\ign{k}$ corresponds to $k-1$ rounds of $\wl{k}$. This is because $\igns{k}$ propagate information to $k$-tuples from all other $k$-tuples, whereas $\wl{k}$ only propagates information from neighbouring $k$-tuples. 

To see this, consider $k=3$ and let $\bar v=(v_1,v_2,v_3)$ be a triple in $(V(G))^3$. When a $\ign{3}$ $F$ applies a layer $L^{(t)}$, the vector $\Bigl(L^{(t)}\bigl(F^{(t-1)}(\mathbf{A}_G)\bigr)\Bigr)_{\bar v,\bullet}$ is computed based on all vectors $\bigl(F^{(t-1)}(\mathbf{A}_G)\bigr)_{\bar v',\bullet}$ for $\bar v'\in (V(G))^3$. For example, $\Bigl(L^{(t)}\bigl(F^{(t-1)}(\mathbf{A}_G)\bigr)\Bigr)_{\bar v,\bullet}$ depends on 
$\bigl(F^{(t-1)}(\mathbf{A}_G)\bigr)_{\bar v',\bullet}$ with  $\bar v'=(v_1,v_2',v_3')$ with $v_2'$ and $v_3'$ being different from $v_1$, $v_2$ and $v_3$.
By contrast, in round $t$, $\wl{3}$ updates the label of $\bar v$ only based on the labels, computed in round $t-1$, of triples of the form $(v_1',v_2,v_3)$, $(v_1,v_2',v_3)$ and $(v_1,v_2,v_3')$ for $v_1',v_2',v_3'\in V(G)$. We observe that the triple $\bar v'$ is not included here and hence the label $\bar v$ is not updated in round $t$ based on the label, computed in round $t-1$, of $\bar v'$. We note, however, that in round $t$, $\wl{3}$ also updates the label of the triple $(v_1,v_2,v_3')$ based on the label, computed in round $t-1$, of $\bar v'=(v_1,v_2',v_3')$ as $\bar v'$ is now one of the neighbours of $(v_1,v_2,v_3')$. As a consequence, in round $t+1$, $\wl{3}$ will update the label of $\bar v$ based on the label, computed in round $t$, of $(v_1,v_2,v_3')$. The latter now depends on the label, computed in round $t-1$, of $\bar v'$. Hence, only in round $t+1$ the label of $\bar v$ includes information about the label, computed in round $t-1$, of $\bar v'$. By contrast, as we have seen earlier, $\Bigl(L^{(t)}\bigl(F^{(t-1)}(\mathbf{A}_G)\bigr)\Bigr)_{\bar v,\bullet}$ immediately takes into account information from  $\bar v'=(v_1,v_2',v_3')$.
We thus see that $\wl{3}$ needs two rounds for a single application of an equivariant layer in a $\ign{3}$. In other words, $t$ rounds of $\wl{3}$ correspond to application of $\lfloor \frac{t}{2}\rfloor$ equivariant layers in an $\ign{3}$. This holds more generally for any $k\geq 2$.

Furthermore, it is thanks to the invariance and equivariance of the layers in $\igns{k}$ that the information propagation happens in a controlled way. More specifically, a $\ign{k}$ propagates information from triples with the same equality pattern in the same way. As we will see shortly, this is crucial for showing Lemmas~\ref{lem:mainpert} and~\ref{lem:fromeqtofull}.

\subsection{Proof of Lemma \ref{lem:mainpert}} 
We show $G\equiv_{\wl{k}}^t H \Rightarrow G \equiv_{F}^{\lfloor\frac{t}{k-1}\rfloor} H$ by induction on $t$. The proof strategy is similar to the one used by~\citet{chen2020graph} except that we rely on a more general key lemma in the inductive step.
As mentioned earlier, we will show a stronger induction hypothesis. More specifically, we show that for any $t$ and
$k$-tuples $\bar v\in (V(G))^k$ and $\bar w\in (V(H))^k$, if $G\equiv_{\wl{k}}^t H$, then 
\begin{equation*}\chi_{G,k}^{(t)}(\bar v)=\chi_{H,k}^{(t)}(\bar w)\Rightarrow\bigl(F^{\left(\lfloor\frac{t}{k-1}\rfloor\right)}(\mathbf{A}_G)\bigr)_{\bar v,\bullet}=\bigl(F^{\left(\lfloor\frac{t}{k-1}\rfloor\right)}(\mathbf{A}_H)\bigr)_{\bar w,\bullet}.\tag{$\ddagger$}\label{eq:indhyp}
\end{equation*}
It is an easy observation that the implication~(\ref{eq:indhyp}) implies $G\equiv_{\wl{k}}^t H \Rightarrow G \equiv_{F}^{\lfloor\frac{t}{k-1}\rfloor} H$. Indeed,
suppose that $G\equiv_{\wl{k}}^t H$ holds. By definition, this is equivalent to 
$$
\Bldbl \chi_{G,k}^{(t)}(\bar v)\bigm| \bar v\in (V(G))^k \Brdbl=
\Bldbl \chi_{H,k}^{(t)}(\bar w)\bigm| \bar w\in(V(H))^k\Brdbl. 
$$
In other words, with every $\bar v\in (V(G))^k$ one can associate a corresponding $\bar w\in (V(H))^k$ such that 
$\chi_{G,k}^{(t)}(\bar v)=\chi_{H,k}^{(t)}(\bar w)$. Then,~(\ref{eq:indhyp}) implies $\bigl(F^{\left(\lfloor\frac{t}{k-1}\rfloor\right)}(\mathbf{A}_G)\bigr)_{\bar v,\bullet}=\bigl(F^{\left(\lfloor\frac{t}{k-1}\rfloor\right)}(\mathbf{A}_H)\bigr)_{\bar w,\bullet}$. Since this holds for any $\bar v\in (V(G))^k$ and its corresponding $\bar w\in(V(H))^k$, we have 
$$
\Bldbl \bigl(F^{\left(\lfloor\frac{t}{k-1}\rfloor\right)}(\mathbf{A}_G)\bigr)_{\bar v,\bullet}\bigm| \bar v\in (V(G))^k \Brdbl=
\Bldbl \bigl(F^{\left(\lfloor\frac{t}{k-1}\rfloor\right)}(\mathbf{A}_H)\bigr)_{\bar w,\bullet}\bigm| \bar w\in(V(H))^k\Brdbl. 
$$
This in turn is equivalent to $G \equiv_{F}^{\lfloor\frac{t}{k-1}\rfloor} H$, by definition.

Furthermore, we observe that it suffices to show~(\ref{eq:indhyp}) for $t$ being a multiple of $k-1$.
Indeed, suppose that $t$ is not a multiple of $k-1$. That is, $t=m(k-1)+r$ for some $m,r\in\Nb$ satisfying $0< r < k-1$. Let us consider $t^\circ=m(k-1)$ and note that $\lfloor \frac{t^\circ}{k-1}\rfloor=\lfloor \frac{m(k-1)}{k-1}\rfloor=m$. 
Suppose that we already have shown~(\ref{eq:indhyp}) for $t^\circ$. It now suffices to observe
that $\chi_{G,k}^{(t)}(\bar v)=\chi_{H,k}^{(t)}(\bar w)$ implies 
$\chi_{G,k}^{(t^\circ)}(\bar v)=\chi_{H,k}^{(t^\circ)}(\bar w)$ since $\wl{k}$ produces refinements of colourings and $t^\circ\leq t$. Because, by assumption,  $\chi_{G,k}^{(t^\circ)}(\bar v)=\chi_{H,k}^{(t^\circ)}(\bar w)$ implies $\bigl(F^{(m)}(\mathbf{A}_G)\bigr)_{\bar v,\bullet}=\bigl(F^{(m)}(\mathbf{A}_H)\bigr)_{\bar w,\bullet}$ and $\lfloor\frac{t}{k-1}\rfloor=\lfloor\frac{m(k-1)+r}{k-1}\rfloor=m$, we may conclude that~(\ref{eq:indhyp}) holds for $t$ as well. In the following we therefore assume that $t=m(k-1)$ for some $m\in\Nb$ with $0\leq m \leq d$. We next show the implication~(\ref{eq:indhyp}).

\paragraph{Base case.}
In this case, $t=0$ and the induction hypothesis is 
$\chi_{G,k}^{(0)}(\bar v)=\chi_{H,k}^{(0)}(\bar w)\Rightarrow \bigl(F^{(0)}(\mathbf{A}_G)\bigr)_{\bar v,\bullet}=\bigl(F^{(0)}(\mathbf{A}_H)\bigr)_{\bar w,\bullet}$. Since $F^{(0)}$ is defined as the identity mapping, we need to verify  $(\mathbf{A}_G)_{\bar v,\bullet}=(\mathbf{A}_H)_{\bar w,\bullet}$.
We note, however, that $\mathbf{A}_G$
and $\mathbf{A}_H$ are defined by hot-one encoding $\chi_{G,k}^{(0)}$ and $\chi_{H,k}^{(0)}$, respectively.
In particular, if $\chi_{G,k}^{(0)}(\bar v)=\chi_{H,k}^{(0)}(\bar w)=c_s$ for $s\in [s_0]$ and $c_s\in\mathcal{C}$ (recall that $s_0$ denotes the number of colours assigned by the initial $\wl{k}$ colouring), then 
$$
(\mathbf{A}_G)_{\bar v,\bullet}=\mathbf{b}_s=(\mathbf{A}_H)_{\bar w,\bullet},
$$
where $\mathbf{b}_s$ is the $s$th basis vector in $\Rb^{s_0}$. In other words, the base case holds.

\paragraph{Inductive case.}
Let $t=m(k-1)$ for some $m\in[d]$ and assume that~(\ref{eq:indhyp}) holds for $t'=(m-1)(k-1)$.
We claim that~(\ref{eq:indhyp}) holds for $t$, provided that we can show the key lemma below. The lemma is shown by a different proof technique than used by \citet{chen2020graph} for $k=2$. More specifically, we leverage the connection between $\wl{k}$ and counting logics. By contrast, \citet{chen2020graph} use a case analysis and combinatorial arguments which do not easily generalise to arbitrary $k$. We defer the proof the lemma to Section~\ref{subsec:proofofkeylemma}.  
\begin{klemma}\label{lem:key}
Let $t=m(k-1)$ and $t'=(m-1)(k-1)$ for $m\in\Nb$ and $m\geq 1$.	
Let $G$ and $H$ be such that $G\equiv_{\wl{k}}^{t'} H$ holds and let $\bar v\in (V(G))^k$	 and $\bar w\in (V(H))^k$ be $k$-tuples satisfying $\chi_{G,k}^{(t)}(\bar v)=\chi_{H,k}^{(t)}(\bar w)$. Then, 
\begin{equation}
\Bldbl \chi_{G,k}^{(t')}(\bar v')\bigm| (\bar v,\bar v')\in\mu \Brdbl=
\Bldbl \chi_{H,k}^{(t')}(\bar w')\bigm|(\bar w,\bar w')\in\mu\Brdbl \tag{$\ddagger\ddagger$}\label{eq:equal2}
\end{equation}
for every equality pattern $\mu\in[n]^{2k}/_\sim$.\qed
\end{klemma}
Intuitively, this lemma allows us to reason over multisets of colours of $k$-tuples grouped together according to an equality pattern. Since each equivariant layer in a $\igns{k}$ treats tuples satisfying the same equality pattern in the same way, the lemma suffices to show the implication~(\ref{eq:indhyp}). In the remainder of this section, we formally verify that the Key Lemma indeed implies the implication~(\ref{eq:indhyp}) for $t=m(k-1)$. 

Let us assume  $G\equiv_{\wl{k}}^t H$ and consider
$k$-tuples $\bar v\in (V(G))^k$ and $\bar w\in (V(H))^k$ satisfying $\chi_{G,k}^{(t)}(\bar v)=\chi_{H,k}^{(t)}(\bar w)$. We need to show  $\bigl(F^{(m)}(\mathbf{A}_G)\bigr)_{\bar v,\bullet}=\bigl(F^{(m)}(\mathbf{A}_H)\bigr)_{\bar w,\bullet}$. We observe that $G\equiv_{\wl{k}}^t H$
implies $G\equiv_{\wl{k}}^{t'} H$ since $t'\leq t$ and $\wl{k}$ produces refinements of colourings. 
As a consequence, the Key Lemma applies. Furthermore, by induction,
for any $\bar v\in (V(G))^k$	 and $\bar w\in (V(H))^k$, if $G\equiv_{\wl{k}}^{t'} H$, then  $\chi_{G,k}^{(t')}(\bar v)=\chi_{H,k}^{(t')}(\bar w)\Rightarrow \bigl(F^{(m-1)}(\mathbf{A}_G)\bigr)_{\bar v,\bullet}=\bigl(F^{(m-1)}(\mathbf{A}_H)\bigr)_{\bar w,\bullet}$. From the equality~(\ref{eq:equal2}) we can now infer 
\begin{equation}
\Bldbl \bigl(F^{(m-1)}(\mathbf{A}_G)\bigr)_{\bar v',\bullet}\bigm| (\bar v,\bar v')\in\mu \Brdbl=
\Bldbl \bigl(F^{(m-1)}(\mathbf{A}_H)\bigr)_{\bar w',\bullet}\bigm| (\bar w,\bar w')\in\mu \Brdbl,\label{eq:fmin1}
\end{equation}
for any $\mu\in[n]^{2k}/_\sim$. 
We recall that $F^{(m)}=\sigma\circ L^{(m)}\circ F^{(m-1)}$. We next use that $L^{(m)}:\Rb^{n^k\times s_{m-1}}\to \Rb^{n^k\times s_m}$ is an equivariant layer
and hence can be decomposed according to equality types $\mu\in[n]^{2k}/_\sim$, as shown in Section~\ref{sec:background}. More specifically, we next show that 
the equality~(\ref{eq:fmin1}) implies 
\begin{equation}
	\Bigl(L^{(m)}_\mu\bigl(F^{(m-1)}(\mathbf{A}_G)\bigr)\Bigr)_{\bar v,\bullet}=\Bigl(L^{(m)}_\mu\bigl(F^{(m-1)}(\mathbf{A}_H)\bigr)\Bigr)_{\bar w,\bullet} \label{eq:lmmu}
\end{equation}
for every $\mu\in[n]^{2k}/_\sim$.
Indeed, let us first recall that for $a\in[s_m]$ and equality pattern $\mu\in[n]^{2k}/_\sim$:
\begin{align*}
\Bigl(L^{(m)}_\mu\bigl(F^{(m-1)}(\mathbf{A}_G)\bigr)\Bigr)_{\bar v,a}	&=
\sum_{\substack{\bar v'\in[n]^k\\ (\bar v,\bar v')\in\mu}}
	\sum_{b\in [s_{m-1}]} c_{\mu,a,b}\bigl(F^{(m-1)}(\mathbf{A}_G)\bigr)_{\bar v',b}\\
		\Bigl(L^{(m)}_\mu\bigl(F^{(m-1)}(\mathbf{A}_H)\bigr)\Bigr)_{\bar w,a}	&=
		\sum_{\substack{\bar w'\in[n]^k\\ (\bar w,\bar w')\in\mu}}\!\!
		\sum_{b\in [s_{m-1}]} c_{\mu,a,b}\bigl(F^{(m-1)}(\mathbf{A}_H)\bigr)_{\bar w',b}.
\end{align*}
It now suffices to observe that the coefficients $c_{\mu,a,b}$ only depend on the equality pattern
$\mu$, $a\in[s_m]$ and $b\in[s_{m-1}]$. From equality~(\ref{eq:fmin1}) we know
that with each $\bar v'$ satisfying $(\bar v,\bar v')\in \mu$
we can associate a unique $\bar w'$ satisfying $(\bar w,\bar w')\in \mu$ such that for each $b\in [s_{m-1}]$,
$$
\bigl(F^{(m-1)}(\mathbf{A}_G)\bigr)_{\bar v',b}=\bigl(F^{(m-1)}(\mathbf{A}_H)\bigr)_{\bar w',b},
$$
and thus also
$$
c_{\mu,a,b}\bigl(F^{(m-1)}(\mathbf{A}_G)\bigr)_{\bar v',b}=c_{\mu,a,b}\bigl(F^{(m-1)}(\mathbf{A}_H)\bigr)_{\bar w',b}
$$
holds.
Given that
$\Bigl(L^{(m)}_\mu\bigl(F^{(m-1)}(\mathbf{A}_G)\bigr)\Bigr)_{\bar v,\bullet}$ and
$\Bigl(L^{(m)}_\mu\bigl(F^{(m-1)}(\mathbf{A}_H)\bigr)\Bigr)_{\bar w,\bullet}$ are defined as the sums over elements $\bar v'$ and $\bar w'$ satisfying $(\bar v,\bar v')\in\mu$ and $(\bar w,\bar w')\in\mu$, respectively, we may conclude that
$\Bigl(L^{(m)}_\mu\bigl(F^{(m-1)}(\mathbf{A}_G)\bigr)\Bigr)_{\bar v,\bullet}=\Bigl(L^{(m)}_\mu\bigl(F^{(m-1)}(\mathbf{A}_H)\bigr)\Bigr)_{\bar w,\bullet}$, as desired.

We next show that equality~(\ref{eq:lmmu}) implies 
\begin{equation}
\Bigl(L^{(m)}\bigl(F^{(m-1)}(\mathbf{A}_G)\bigr)\Bigr)_{\bar v,\bullet}=\Bigl(L^{(m)}\bigr(F^{(m-1)}(\mathbf{A}_H)\bigr)\Bigr)_{\bar w,\bullet}.\label{eq:lm}
\end{equation}
Indeed, we  recall that for $a\in[s_m]$:
\begin{align*}
\Bigl(L^{(m)}\bigl(F^{(m-1)}(\mathbf{A}_G)\bigr)\Bigr)_{\bar v,a}&=\sum_{\mu\in [n]^{2k}/_\sim} 
\Bigl(L_\mu^{(m)}\bigl(F^{(m-1)}(\mathbf{A}_G))\bigr)\Bigr)_{\bar v,a}+
c_{\tau, a}\\
\Bigl(L^{(m)}\bigl(F^{(m-1)}(\mathbf{A}_H)\bigr)\Bigr)_{\bar w,a}&=\sum_{\mu\in [n]^{2k}/_\sim} \Bigl(L_\mu^{(m)}\bigl(F^{(m-1)}(\mathbf{A}_H)\bigr)\Bigr)_{\bar w,a}+
c_{\tau',a}
\end{align*}	
where $\tau,\tau'\in[n]^k/_\sim$ and $\bar v\in\tau$ and $\bar w\in\tau'$. Clearly,~(\ref{eq:lmmu}) implies~(\ref{eq:lm})  if we can show that $\tau=\tau'$ and thus $c_{\tau,a}=c_{\tau',a}$ for all $a\in[s_m]$. Stated differently, we need to show that $\bar v\sim \bar w$. This is, however, a direct  consequence of the assumption $\chi_{G,k}^{(t)}(\bar v)=\chi_{H,k}^{(t)}(\bar w)$. Indeed, $\chi_{G,k}^{(t)}(\bar v)=\chi_{H,k}^{(t)}(\bar w)$ implies 
$\chi_{G,k}^{(0)}(\bar v)=\chi_{H,k}^{(0)}(\bar w)$, which in turn implies that $\bar v$ and $\bar w$ have the same isomorphism type. In particular, 
$v_i=v_j\Leftrightarrow w_i=w_j$ for all $i,j\in[k]$. As a consequence, $\bar v$ and $\bar w$ have the same equality pattern. \looseness=-1

To conclude the proof, it remains to show 
$\bigl(F^{(m)}(\mathbf{A}_G)\bigr)_{\bar v,\bullet}=\bigl(F^{(m)}(\mathbf{A}_H)\bigr)_{\bar w,\bullet}$. 
We recall again that $F^{(m)}=\sigma\circ L^{(m)} \circ F^{(m-1)}$ and hence, due to the equality~(\ref{eq:lm}) it suffices to observe that~(\ref{eq:lm}) remains to true after applying the activation function $\sigma$. We recall that such an activation function $\sigma$ is defined in a pointwise manner. That is, for a vector
$\bar a\in\Rb^q$, $\sigma(\bar a)=(\sigma(a_1),\ldots,\sigma(a_q))$. More generally, for a tensor $\mathbf{A}\in\Rb^{n^k\times q}$ and $\bar v\in[n]^k$,
$\bigl(\sigma(\mathbf{A})\bigr)_{\bar v,\bullet}=\sigma(\mathbf{A}_{\bar v,\bullet})$. Hence, the equality~(\ref{eq:lm}) indeed implies
\begin{align*}
\biggl(\sigma\Bigl(L^{(m)}\bigl(F^{(m-1)}(\mathbf{A}_G)\bigr)\Bigr)\biggr)_{\bar v,\bullet}&=
\sigma\Bigl(L^{(m)}\bigl(F^{(m-1)}(\mathbf{A}_G)\bigr)_{\bar v,\bullet}\Bigr)\\
&=\sigma\Bigl(L^{(m)}\bigl(F^{(m-1)}(\mathbf{A}_H)\bigr)_{\bar w,\bullet}\Bigr)\\
&=\biggl(\sigma\Bigl(L^{(m)}\bigl(F^{(m-1)}(\mathbf{A}_H)\bigr)\Bigr)\biggr)_{\bar w,\bullet},
\end{align*}
from which $\bigl(F^{(m)}(\mathbf{A}_G)\bigr)_{\bar v,\bullet}=\bigl(F^{(m)}(\mathbf{A}_H)\bigr)_{\bar w,\bullet}$ follows, as desired.\hfill\qed

\subsection{Proof of Lemma \ref{lem:fromeqtofull}} 
Let $t=d(k-1)$. We show that if for any two $\bar v\in (V(G))^k$ and $\bar w\in (V(H))^k$, we have $\chi_{G,k}^{(t)}(\bar v)= \chi_{H,k}^{(t)}(\bar w)\Rightarrow \bigl(F^{(d)}(\mathbf{A}_G)\bigr)_{\bar v,\bullet}=
\bigl(F^{(d)}(\mathbf{A}_H)\bigr)_{\bar w,\bullet}$, then
$	G\equiv_{\wl{k}}^t H\Rightarrow G\equiv_F H$ holds.

We assume that $G\equiv_{\wl{k}}^t H$ holds for $t=d(k-1)$. 
By definition, this implies
\begin{equation}
	\Bldbl \chi_{G,k}^{(t)}(\bar v)\bigm| \bar v\in (V(G))^k \Brdbl=
\Bldbl \chi_{H,k}^{(t)}(\bar w)\bigm|\bar w\in (V(H))^k\Brdbl. \label{eq:wlk}
\end{equation}
Furthermore, we observe that $\chi_{G,k}^{(t)}(\bar v)= \chi_{H,k}^{(t)}(\bar w)\Rightarrow
\chi_{G,k}^{(0)}(\bar v)= \chi_{H,k}^{(0)}(\bar w)$. As observed earlier, this implies that $\bar v\sim\bar w$.
In other words, $\bar v$ and $\bar w$
have the same equality pattern $\tau\in[n]^k/_\sim$. As a consequence, together with~(\ref{eq:wlk})
this implies that for every $\tau\in [n]^k/_\sim$,
\begin{equation}
	\Bldbl \chi_{G,k}^{(t)}(\bar v)\bigm| \bar v \in \tau, \bar v\in (V(G))^k \Brdbl=
\Bldbl \chi_{H,k}^{(t)}(\bar w)\bigm|\bar w\in \tau, \bar w\in (V(H))^k\Brdbl. \label{eq:wlkpertau}
\end{equation}
We further assume that for
$\bar v\in (V(G))^k$ and $\bar w\in (V(H))^k$, $\chi_{G,k}^{(t)}(\bar v)= \chi_{H,k}^{(t)}(\bar w)\Rightarrow \bigl(F^{(d)}(\mathbf{A}_G)\bigr)_{\bar v,\bullet}=
\bigl(F^{(d)}(\mathbf{A}_H)\bigr)_{\bar w,\bullet}$. 
Hence,~(\ref{eq:wlkpertau}) implies
\begin{equation}
\Bldbl \bigl(F^{(m)}(\mathbf{A}_G)\bigr)_{\bar v,\bullet}\bigm| \bar v\in\tau \Brdbl=
\Bldbl \bigl(F^{(m)}(\mathbf{A}_H)\bigr)_{\bar w,\bullet}\bigm| \bar w\in \tau \Brdbl
\label{eq:Fmpertau}
\end{equation}
for every equality pattern $\tau\in[n]^k/_\sim$.

We now recall that
$F=M\circ I\circ F^{(d)}$ and we need to show that $F(\mathbf{A}_G)=F(\mathbf{A}_H)$.
It suffices to show that 
$I\bigl(F^{(d)}(\mathbf{A}_G)\bigr)=I\bigl(F^{(d)}(\mathbf{A}_H)\bigr)$ since $M$ is an $\mlp$ which encodes a
function from $\Rb^{s_{d+1}}\to \Rb^{s}$. We recall that $I$ is an invariant layer from $\Rb^{n^k\times s_d}$
to $\Rb^{s_{d+1}}$. Since invariant layers are a special case of equivariant layers, they can again be decomposed based on equality patterns. More specifically, for a tensor $\mathbf{A}\in\Rb^{n^k\times s_d}$ and $a\in[s_{d+1}]$,
$$
I(\mathbf{A})_{a}=\sum_{\tau\in [n]^{k}/_\sim} I_\tau(\mathbf{A})_{a} + c_a  \text{ with } 
	I_\tau(\mathbf{A})_{a}=\sum_{\substack{\bar v'\in[n]^k\\\bar v'\in\tau}}
\sum_{b\in [s_d]} c_{\tau,a,b}\mathbf{A}_{\bar v',b}.$$
Then, just as in the proof of Lemma~\ref{lem:mainpert}, when $I$ is applied to $F^{(m)}(\mathbf{A}_G)$ and $F^{(m)}(\mathbf{A}_H)$, and by observing that the constants $c_{\tau,a,b}$ only depend on $\tau$, $a$ and $b$,
we can conclude from~(\ref{eq:Fmpertau}) that $\Bigr(I\bigl(F^{(m)}(\mathbf{A}_G)\bigr)\Bigr)_a=\Bigl(I\bigl(F^{(m)}(\mathbf{A}_H)\bigr)\Bigr)_a$ for all $a\in[s_{d+1}]$. In other words, $I\bigl(F^{(d)}(\mathbf{A}_G)\bigr)=I\bigl(F^{(d)}(\mathbf{A}_H)\bigr)$ and thus $G\equiv_F H$, as desired.\hfill\qed

\subsection{Proof of the key lemma}\label{subsec:proofofkeylemma}
Let $t=m(k-1)$ and $t'=(m-1)(k-1)$. We recall that the Key Lemma requires us to show that if $\bar v$ and $\bar w$ satisfy  $\chi_{G,k}^{(t)}(\bar v)=\chi_{H,k}^{(t)}(\bar w)$ and  if $G\equiv_{\wl{k}}^{t'} H$ holds, then
\begin{equation*}
\Bldbl \chi_{G,k}^{(t')}(\bar v')\bigm| (\bar v,\bar v')\in\mu \Brdbl=
\Bldbl \chi_{H,k}^{(t')}(\bar w')\bigm| (\bar w,\bar w') \in \mu\Brdbl \tag{$\ddagger\ddagger$}
\end{equation*}
for any equality pattern $\mu\in[n]^{2k}/_\sim$. 

We will show the equality~($\ddagger\ddagger$) by assuming, for the sake of contradiction, that there exists an equality pattern $\mu$ for which equality~($\ddagger\ddagger$) does not hold. For such a pattern $\mu$, and $k$-tuples $\bar v\in (V(G))^k$ and
$\bar w\in (V(H))^k$, we 
then construct a 
 formula $\varphi(x_1,\ldots,x_k)$ in $\countinglogic{k}$ of quantifier rank at most $t$, such that 
$G\models\varphi[\bar v]$ but $H\not\models\varphi[\bar w]$.
This contradicts $\chi_{G,k}^{(t)}(\bar v)=\chi_{H,k}^{(t)}(\bar w)$
as this implies that $\bar v$ and $\bar w$ satisfy the same formulas in 
 $\countinglogic{k}$ of quantifier rank at most $t$ (cfr.~Theorem~\ref{thm:cfi}). In other words, no equality pattern $\mu$ can exist that violates ($\ddagger\ddagger$).
There will be some special equality patterns for which no formula can be constructed. We treat these cases separately using the assumption $G\equiv_{\wl{k}}^{t'} H$ instead.\looseness=-1

We start by introducing some concepts related to equality patterns. Let $\mu\in [n]^{2k}/_\sim$ and let $\bar v=(v_1,\ldots,v_k)\in (V(G))^k$ and 
$\bar v'=(v_1',\ldots,v_k')\in (V(G))^k$. We represent $\mu$ by its partition $[2k]=I_1\uplus\cdots\uplus I_r$.
For a class $I_s$, with $s\in[r]$, we define $\mathsf{rep}(I_s)$ as the smallest index $i$ in $I_s$. We now distinguish between different kinds of classes. A class $I_s$ is called constant if $\mathsf{rep}(I_s)\leq k$.
 When $\mathsf{rep}(I_s)>k$ we call $I_s$
variable. 
Among constant classes, we further distinguish been constant classes that are used, and those that are not. A constant class $I_s$ is called used when it contains entries strictly larger than $k$. Intuitively, indexes $i>k$ in a used constant class $I_s$ indicate that for $(\bar v,\bar v')$ to be in $\mu$, $v_{i-k}'=v_{\mathsf{rep}(I_s)}$. In other words, those entries in $\bar v'$ take values from $\bar v$. Unused constant classes represent entries in $\bar v$ that must be different from any entry in $\bar v'$.

For notational convenience we introduce $P_{\mu,\bar v}:=\{\bar v'\in (V(G))^k\mid (\bar v,\bar v')\in\mu\}$ and similarly, $Q_{\mu,\bar w}:=\{\bar w'\in (V(H))^k\mid (\bar w,\bar w')\in \mu\}$.
It will be useful to rephrase $\bar v'\in P_{\mu,\bar v}$ in terms of equality and inequality conditions relative to the partition $[2k]=I_1\uplus\cdots\uplus I_r$ of $\mu$. More specifically, $\bar v'\in P_{\mu,\bar v}$ if and only if:
$$\begin{cases}
v'_i= v'_j & \text{for $k+i,k+j\in I_s$, where $I_s$ is a variable or a used constant class;} \hfill\text{(a)}\\
v'_i\neq v'_j & \text{for $k+i\in I_s$, $k+j\in I_{s'}$, $s\neq s'$, where $I_s$ and $I_{s'}$ are either}\\
& \hspace{6.1cm}\hfill \text{variable or used constant classes;}  \hspace{1ex}\hfill\text{(b)}\\
v'_i= v_{\mathsf{rep}(I_s)} & \text{$k+i\in I_s$, where $I_s$ is a used constant class; and} \hfill \text{(c)}\\
v'_i\neq v_{\mathsf{rep}(I_{s'})} & \text{$k+i\in I_{s}$, where $I_{s}$ is a variable class and $I_{s'}$ is a constant but unused class.}\hfill\text{(d)}
\end{cases} 
$$
That is, condition (a) simply states which entries in $\bar v'$ must be the same and condition (c) tells which entries in $\bar v'$ take values from entries in $\bar v$. Moreover, condition (b) states which entries in $\bar v'$ are distinct from each other.
These conditions together imply that any entry in $\bar v'$ belonging to a variable class is necessarily distinct from 
entries in $\bar v$ belonging to a used constant class.  Finally, condition (d) states that  any entry in $\bar v'$ belonging to a variable class should also be distinct from entries in $\bar v$ belonging to an unused constant class.  With this notation, we can rephrase equality~($\ddagger\ddagger$) as
\begin{equation}
\Bldbl \chi_{G,k}^{(t')}(\bar v')\bigm| \bar v'\in P_{\mu,\bar v} \Brdbl=
\Bldbl \chi_{H,k}^{(t')}(\bar w')\bigm| \bar w' \in Q_{\mu,\bar w}\Brdbl,\label{eq:Pset}
\end{equation}
where $\chi_{G,k}^{(t)}(\bar v)=\chi_{H,k}^{(t)}(\bar w)$.
Directly applying our proof strategy, using formulas in $\countinglogic{k}$ of quantifier rank at most $t$, to $k$-tuples in $P_{\mu,\bar v}$ and $Q_{\mu,\bar w}$, is problematic, however, as is illustrated in the following example.

\begin{example}\normalfont
Let $k=3$ and consider the equality pattern $\mu\in[n]^6/_\sim$ represented by $[6]=I_1\uplus I_2\uplus I_3\uplus I_4\uplus I_5$ with 
$I_1:=\{1,4\}$, $I_2:=\{2\}$, $I_3:=\{3\}$, $I_4:=\{5\}$ and $I_6:=\{6\}$. We remark that $I_1$ is the only used constant class with $\mathsf{rep}(I_1)=1$.
The unused constant classes are $I_2$ and $I_3$, and the variables classes are
$I_4$ and $I_5$. For a six-tuple $(\bar v,\bar v')$ to be in $\mu$, all entries in $\bar v=(v_1,v_2,v_3)$ must be pairwise distinct and $\bar v'=(v_1',v_2',v_3')$ is of the form $(v_1,v_2',v_3')$ with $v_2'\neq v_3'$ and
$v_2'$ and $v_3'$ distinct from $v_1$, $v_2$ and $v_3$. Suppose that the equality~(\ref{eq:Pset}) does not hold for our example $\mu$. Assume, for example, that there are more than $m$ triples in $P_{\mu,\bar v}$ of colour $c'$, assigned by $\wl{3}$ in round $t'$, whereas $Q_{\mu,\bar w}$ has less than $m$ such triples. By assumption, we have that $\chi_{G,3}^{(t)}(\bar v)=\chi_{H,3}^{(t)}(\bar w)$ and let us assume that $\wl{3}$ assigns colour $c$ in round $t$ to both these triples. We now intend to use a formula in $\countinglogic{3}$ of quantifier rank at most $t$ that allows us to distinguish $\bar v$ from $\bar w$. As previously mentioned, if we can find such a formula, then we obtain a contradiction to our assumption 
$\chi_{G,3}^{(t)}(\bar v)=\chi_{H,3}^{(t)}(\bar w)$. A candidate formula would be one that is satisfied for any triple $\bar v$ of colour $c$, assigned by $\wl{3}$ in round $t$, and for which there are more than $m$ triples in $P_{\mu,\bar v}$ of colour $c'$, assigned by $\wl{3}$ in round $t'$.
Indeed, by assumption, $\bar v$ would satisfy this formula whereas $\bar w$ would not. To express this as a logical formula one can consider $\varphi(x_1,x_2,x_3)$ defined as
$$
\psi_{c}^{(t)}(x_1,x_2,x_3)\land \Bigl(
\exists^{\geq m}(x_2',x_3')\, \psi_{c'}^{(t')}(x_1,x_2',x_3') \land
x_2'\neq x_3' \land \bigwedge_{i\in[3]} \bigl( x_i\neq x_2' \land x_i\neq x_3'\bigr)
\Bigr),
$$
where $\psi_{c}^{(t)}$ and $\psi_{c'}^{(t')}$ are  $\countinglogic{3}$ formulas expressing that a tuple is assigned colour $c$ and $c'$ by $\wl{3}$ in round $t$ and $t'$, respectively. We note, however, that we use five variables because we need to ensure that $x_2'$ and $x_3'$ are distinct from $x_1$, $x_2$ and $x_3$. What can easily be expressed using three variables, however, is the following:
$$
\varphi(x_1,x_2,x_3):=
\psi_{c}^{(t)}(x_1,x_2,x_3)\land \Bigl(
\exists^{\geq m}(x_2,x_3)\, \psi_{c'}^{(t')}(x_1,x_2,x_3) \land
x_2\neq x_3 \land  x_1\neq x_2 \land x_1\neq x_3\Bigr).
$$
Here, we reused the variables $x_2$ and $x_3$ and require them to be distinct from each other, as before, but now only require them to be distinct from $x_1$, the free variable in the second conjunct. \hfill\qed
\end{example}

As the example shows, we can easily encode (in-)equalities between reused variables and free variables. Intuitively, the free variables correspond to positions belonging to constant used classes. So, instead of considering 
$k$-tuples in $P_{\mu,\bar v}$ and $Q_{\mu,\bar w}$, it seems feasible to detect differences in the number of occurrences of colours of multisets defined in terms if equality and inequality conditions unrelated to unused constant classes. That is, when the condition (d), part of the characterisation of tuples in  $P_{\mu,\bar v}$ and $Q_{\mu,\bar w}$ mentioned earlier, is ignored.

We thus define $\tilde{P}_{\mu,\bar v}$ as $P_{\mu,\bar v}$ but drop condition (d) from the conditions stated above. That is,
$$
\tilde{P}_{\mu,\bar v}:=\{ \bar v'\in (V(G))^k\mid \text{$\bar v'$ satisfies conditions (a), (b) and (c)}\}.
$$
We define $\tilde{Q}_{\mu,\bar v}$ in a similar way. We next show that we can use these sets of tuples to detect whether or not equality~(\ref{eq:Pset}) holds. More precisely, we  show that we can rewrite  $P_{\mu,\bar v}$ in terms of $\tilde{P}_{\mu',\bar v}$ for some patterns $\mu'$, as we will illustrate next.

\begin{example}\normalfont
		For our example $\mu$, consider the variable class $I_4$ and unused constant class $I_2$. Then, we consider 
$\mu[4\mapsto 2]$  represented by $[6]=\{1,4\}\uplus \{2,5\}\uplus \{3\}\uplus \{6\}$,
 	where $\{2,5\}$ is the result of merging $I_4$ and $I_2$ of $\mu$. We note that
 	$$\tilde{P}_{\mu[4\mapsto 2],\bar v}=\{(v_1,v_2,v_3')\in (V(G))^3\mid \text{$v_3'$ is different from $v_1$ and $v_2$}\}.$$
 	We can similarly consider other pairs of variable and unused constant classes. More specifically, we can consider
		$\mu[4\mapsto 3]$, $\mu[5\mapsto 2]$ and $\mu[5\mapsto 3]$ resulting in
 	\begin{align*}
 		\allowdisplaybreaks
 	\tilde{P}_{\mu[4\mapsto 3],\bar v}&:=\{(v_1,v_3,v_3')\in (V(G))^3\mid \text{$v_3'$ is different from $v_1$ and $v_3$}\}\notag\\
 	\tilde{P}_{\mu[5\mapsto 2],\bar v}&:=\{(v_1,v_2',v_2)\in (V(G))^3\mid \text{$v_2'$ is different from $v_1$ and $v_2$}\}\notag\\
 	\tilde{P}_{\mu[5\mapsto 3],\bar v}&:=\{(v_1,v_2',v_3)\in (V(G))^3\mid \text{$v_2'$ is different from $v_1$ and $v_3$}\}. 
 	\end{align*}
	It is now readily verified that 
	\begin{equation*}
		P_{\mu,\bar v}=\tilde{P}_{\mu,\bar v} \setminus \bigl(\tilde{P}_{\mu[4\mapsto 2],\bar v} \cup
		\tilde{P}_{\mu[4\mapsto 3],\bar v} \cup
	\tilde{P}_{\mu[5\mapsto 2],\bar v} \cup	
	\tilde{P}_{\mu[5\mapsto 3],\bar v}\bigr) \tag*{\qed}
	\end{equation*}
\end{example}

The rewriting of $P_{\mu,\bar v}$ in terms of $\tilde{P}_{\mu',\bar v}$ in the previous example holds in general.
\begin{observation}\label{obs:decomp}
	Let $\mu\in[n]^{2k}/_\sim$ be an equality pattern and let $[2k]=I_1\uplus\cdots\uplus I_r$ be its corresponding partition.
	Then,
	$$P_{\mu,\bar v}=\tilde{P}_{\mu,\bar v} \setminus \left(\bigcup_{s,s'} \tilde{P}_{\mu[s\mapsto s'],\bar v}\right)$$
	where $s$ ranges over variables classes $I_s$ and $s'$ ranges over unused constant classes $I_{s'}$.
\end{observation}
\begin{proof}
	We first consider the inclusion $P_{\mu,\bar v}\subseteq \tilde{P}_{\mu,\bar v} \setminus \left(\bigcup_{s,s'} \tilde{P}_{\mu[s\mapsto s'],\bar v}\right)$. Let $\bar v'\in P_{\mu,\bar v}$. This implies that $\bar v'$ satisfies
	conditions (a), (b), (c) and (d) relative to $I_1\uplus\cdots\uplus I_r$. We remark that 
	$\bar v'\in \tilde{P}_{\mu,\bar v}$, simply because the latter is defined in terms of conditions (a), (b) and (c) only.
	Suppose, for the sake of contradiction, that there exists a variable class $I_s$ and an unused constant class $I_{s'}$ such that 
	$\bar v'\in\tilde{P}_{\mu[s\mapsto s'],\bar v}$. This implies that $\bar v'$ satisfies conditions (a), (b) and (c) relative to the partition $[2k]=I_{1}\uplus \cdots \uplus I_{s'-1}\uplus I_{s'+1}\uplus \cdots \uplus I_{s-1}\uplus
	I_{s+1}\uplus\cdots\uplus I_r \uplus (I_{s'}\cup I_{s})$, where  $I_{s'}\cup I_{s}$ is now a used constant  class for
	$\mu[s\mapsto s']$. Condition (c) then implies that for $k+i\in I_s$, $v'_i=v_{\mathsf{rep}(I_{s'})}$.
	This, however, contradicts that $\bar v'$ satisfies condition (d) relative to $I_1\uplus\cdots\uplus I_r$. In other words, $\bar v'\not\in  \tilde{P}_{\mu[s\mapsto s'],\bar v}$. Hence, $\bar v'\in \tilde{P}_{\mu,\bar v} \setminus \left(\bigcup_{s,s'} \tilde{P}_{\mu[s\mapsto s'],\bar v}\right)$ and the inclusion follows.
	
	For the other direction, i.e., to show $\tilde{P}_{\mu,\bar v} \setminus \left(\bigcup_{s,s'} \tilde{P}_{\mu[s\mapsto s'],\bar v}\right)\subseteq P_{\mu,\bar v}$, we argue in a similar way. Consider $\bar v'\in\tilde{P}_{\mu,\bar v} \setminus \left(\bigcup_{s,s'} \tilde{P}_{\mu[s\mapsto s'],\bar v}\right)$. Since $\bar v'\in\tilde{P}_{\mu,\bar v}$, this implies that $\bar v'$ satisfies conditions (a), (b) and (c) relative to $I_1\uplus\cdots\uplus I_r$. If we can show that $\bar v'$ also satisfies condition (d) then $\bar v'\in P_{\mu,\bar v}$, 
	as desired. Suppose, for the sake of contradiction, that $\bar v'$ does not satisfy condition (d) relative to
	$I_1\uplus\cdots\uplus I_r$. This implies that there exists a variable  class $I_s$ and an unused constant class $I_{s'}$ such that
    for $k+i\in I_s$,
	$v_i'=v_{\mathsf{rep}(I_{s'})}$. We now argue that $\bar v'\in \tilde{P}_{\mu[s\mapsto s'],\bar v}$, contradicting our assumption. It suffices to verify that $\bar v'$ satisfies conditions (a), (b) and (c) relative to the partition
	$[2k]=I_{1}\uplus \cdots \uplus I_{s'-1}\uplus I_{s'+1}\uplus \cdots \uplus I_{s-1}\uplus
		I_{s+1}\uplus\cdots\uplus I_r \uplus (I_{s'}\cup I_{s})$ corresponding to $\mu[s\mapsto s']$.
 For condition (a), we only need  to consider the new used constant  class $I_{s'}\cup I_s$ since all other used constant  classes in $\mu[s\mapsto s']$ are used constant  classes for $\mu$, for which condition (a) is already satisfied since $\bar v'\in \tilde{P}_{\mu,\bar v}$. Similarly, each variable class for  
		 $\mu[s\mapsto s']$ is equal to a variable class for $\mu$, so condition (a) holds for those already. Hence, we can focus on $I_{s'}\cup I_s$. Take
		 elements $k+i$ and $k+j$ in $I_{s'}\cup I_{s}$. Since $I_{s'}$ only contains elements smaller or equal than $k$ (it is an unused constant class for $\mu$), $k+i,k+j\in I_{s}$. By assumption, $v_i'=v_{\mathsf{rep}(I_{s'})}=v_j'$ and hence condition (a) is satisfied. We remark that this also shows that condition (c) is satisfied for the new used constant  class $I_{s'}\cup I_s$. For condition (b), we need to compare $I_{s'}\cup I_s$ with used constant or variable classes $I_{s''}$. Assume that $I_{s''}$ is a used constant class. We need to show that for any $k+i\in I_{s''}$ and $k+j\in I_{s'}\cup I_s$, $v_i'\neq v_j'$. We note again that $k+j\in I_{s}$. Since $I_{s}$ is a variable class for $\mu$, $\bar v'\in \tilde{P}_{\mu,\bar v}$ and condition (c)
		 is satisfied for $I_1\uplus\cdots\uplus I_r$, $v_i'\neq v_j'$. Suppose next that $I_{s''}$ is a used constant  class. Then, we know that $\mathsf{rep}(I_{s''})\neq \mathsf{rep}(I_{s'})$ and, since for any $k+j\in I_{s}$,
		 $v_j'=v_{\mathsf{rep}(I_{s'})}$,  we have $v_i'=v_{\mathsf{rep}(I_{s''})}\neq  v_j'=v_{\mathsf{rep}(I_{s'})}$ for any $k+i\in I_{s''}$. Hence, $\bar v'\in \tilde{P}_{\mu[s\mapsto s'],\bar v}$, contradicting our assumption. In other words, $\bar v\in P_{\mu,\bar v}$, as desired, and the inclusion follows.
\end{proof}
We note that all of the above holds for $Q_{\mu,\bar w}$ as well. 

We thus have reduced checking equality~(\ref{eq:Pset}) to checking
\begin{equation}\Bldbl \chi_{G,k}^{(t')}(\bar v')\bigm| \bar v'\in \tilde{P}_{\mu,\bar v} \Brdbl=
\Bldbl \chi_{H,k}^{(t')}(\bar w')\bigm| \bar w'\in \tilde{Q}_{\mu,\bar w} \Brdbl,   \label{eq:tildeP}
\end{equation}
for $\bar v\in (V(G))^k$ and $\bar w\in (V(H))^k$ satisfying $\chi_{G,k}^{(t)}(\bar v)=\chi_{H,k}^{(t)}(\bar w)$, and for any equality pattern $\mu\in[n]^{2k}/_\sim$. To use our proof strategy to detect differences in the number of occurrences of colours of $k$-tuples in $\tilde{P}_{\mu,\bar v}$ and $\tilde{Q}_{\mu,\bar w}$ by means of formulas in $\countinglogic{k}$ of quantifier rank at most $t$, we need to overcome one last hurdle, as is illustrated next.

\begin{example}\normalfont
Let $k=3$ and consider the equality pattern $\mu$ represented by
$[6]=I_1\uplus I_2\uplus I_3\uplus I_4\uplus I_5$ with
$I_1:=\{1,5\}$, $I_2:=\{2\}$, $I_3:=\{3\}$, $I_4:=\{4\}$ and $I_{6}:=\{6\}$.
Consider $\bar v=(v_1,v_2,v_3)$ with all its entries pairwise distinct.
For $\bar v'=(v_1',v_2',v_3')$ to be in $\tilde{P}_{\mu,\bar v}$ it has to be
of the form $(v_1',v_1,v_3')$ with $v_1'$ and $v_3'$ pairwise distinct and distinct from $v_1$. Similarly for $\tilde{Q}_{\mu,\bar w}$ with $\bar w=(w_1,w_2,w_3)$ with all its entries pairwise distinct. Assume that $\bar v$
and $\bar w$ are assigned colour $c$ by $\wl{3}$ in round $t$. Suppose that
the equality~(\ref{eq:tildeP}) does not hold for the equality pattern $\mu$ and triples
$\bar v$ and $\bar w$. In particular, we assume again that there are more than $m$ triples in $\tilde{P}_{\mu,\bar v}$ of colour $c'$, assigned by $\wl{3}$ in round $t'$, whereas there are less than $m$ such triples in $\tilde{Q}_{\mu,\bar w}$.
To express this as a logical formula, we can consider:
\begin{multline*}
\varphi(x_1,x_2,x_3):=
\psi_{c}^{(t)}(x_1,x_2,x_3)\land \Bigl(
\exists^{\geq m}(x_1',x_2,x_3)\, \psi_{c'}^{(t')}(x_1',x_2,x_3)\\{} \land
  x_3\neq x_1' \land x_1'\neq x_1 \land x_3\neq x_1\land x_2=x_1\Bigr).
 \end{multline*}
We note, however that we use four variables because we cannot reuse $x_1$ as it needs to be identified with the reused variable $x_2$.\hfill\qed
\end{example}
In order to avoid having to introduce new variables, as in the previous example, we will replace $\tilde{P}_{\mu,\bar v}$ by a permuted version.
Let $\pi$ be a permutation of $[k]$. For an equality pattern
 $\mu\in[n]^{2k}/_\sim$ represented by $I_1\uplus\cdots\uplus I_r$ we define $\pi\star\mu$ as the equality
 pattern in $[n]^{2k}/_\sim$ represented by $\pi\star I_1\uplus\cdots\uplus\pi\star I_r$, where 
 $\pi\star I_s=\{ \pi(i)\mid i\in I_s, i\leq k\}\cup\{ i\in I_s\mid i>k\}$. 
Furthermore, for a $k$-tuple $\bar v=(v_1,\ldots,v_k)$, we define $\pi\star\bar v:=(v_{\pi^{-1}(1)},\ldots,v_{\pi^{-1}(k)})$ and similarly for $\bar w$ and $\pi\star\bar w$. 
 
We first observe that $\chi_{G,k}^{(t)}(\bar v)=\chi_{H,k}^{(t)}(\bar w)$ implies $\chi_{G,k}^{(t)}(\pi\star\bar v)=\chi_{H,k}^{(t)}(\pi\star\bar w)$ for any permutation
$\pi$ of $[k]$. This is a direct consequence of the fact that $\bar v$ and $\bar w$ satisfy the same formulas in $\countinglogic{k}$ of quantifier rank at most $t$.

\begin{observation}\label{obs:permute}
If	$\chi_{G,k}^{(t)}(\bar v)=\chi_{H,k}^{(t)}(\bar w)$, then also $\chi_{G,k}^{(t)}(\pi\star\bar v)=\chi_{H,k}^{(t)}(\pi\star\bar w)$ for any permutation $\pi$ of $[k]$.
\end{observation}
\begin{proof}
Consider a permutation $\pi:[k]\to[k]$ and suppose, for the sake of contradiction, that
$\chi_{G,k}^{(t)}(\pi\star\bar v)=c'$ and
$\chi_{H,k}^{(t)}(\pi\star \bar w)=c''$ with $c',c''\in\mathcal{C}$ and $c'\neq c''$.
Let $\psi_{c'}^{(t)}(x_1,\ldots,x_k)$ be the $\countinglogic{k}$ formula characterising that
$\wl{k}$ assigns colour $c'$ to $k$-tuples in round $t$. We have that $G\models\psi_{c'}^{(t)}[\pi\star \bar v]$
but $H\not\models\psi_{c'}^{(t)}[\pi\star \bar w]$.
Consider now the formula
$$\pi\star\psi_{c'}^{(t)}(x_1,\ldots,x_k):=\psi_{c'}^{(t)}[x_1/x_{\pi^{-1}(1)},\ldots,x_k/x_{\pi^{-1}(k)}]
$$
obtained from $\psi_{c'}^{(t)}$ by renaming variable $x_i$ by $x_{\pi^{-1}(i)}$. This is again a formula in $\countinglogic{k}$ of quantifier rank at most $t$. Clearly, $G\models\pi\star\psi_{c'}^{(t)}[\bar v]$ if and only if $G\models\psi_{c'}^{(t)}[\pi\star \bar v]$. Similarly, $H\models\pi\star\psi_{c'}^{(t)}[\bar w]$ if and only if $H\models\psi_{c'}^{(t)}[\pi\star \bar w]$.
We may thus conclude that $G\models\pi\star\psi_{c'}^{(t)}[\bar v]$ and $H\not\models\pi\star\psi_{c'}^{(t)}[\bar w]$, contradicting our assumption that  $\chi_{G,k}^{(t)}(\bar v)=\chi_{H,k}^{(t)}(\bar w)$ and thus $\bar v$ and $\bar w$ must
satisfy the same formulas in $\countinglogic{k}$ of quantifier rank at most $t$.
\end{proof}

\begin{remark}\normalfont
	For $k=2$, the observation tells us that $\chi_{G,2}^{(t)}(v_1,v_2)=\chi_{H,2}^{(t)}(w_1,w_2)$ implies 
	$\chi_{G,2}^{(t)}(v_2,v_1)=\chi_{H,2}^{(t)}(w_2,w_1)$. \citet{chen2020graph} infer this by assuming that the graph is undirected. We see, however, that this assumption is not necessary.\hfill\qed
	\end{remark}

We next illustrate how  the permuted versions of $\mu$, $\bar v$ and $\bar w$ come in handy.
\begin{example}\normalfont
	Continuing with the previous example, let $\pi:[3]\to[3]$ be the permutation $1\mapsto 2$, $2\mapsto 1$ and $3\mapsto 3$. Note that $(v_2,v_1,v_3)=\pi\star\bar v$ and 
	$(w_2,w_1,w_3)=\pi\star\bar w$. 
		 Consider the permuted equality pattern $\pi\star\mu$ represented by $\pi\star I_1=\{2,5\}$, $\pi\star I_2=\{1\}$, $\pi\star I_3=\{3\}$, $\pi\star I_4=\{4\}$ and $\pi\star I_5=\{6\}$. Then, for $\bar v'$ to be in 
		 $\tilde{P}_{\pi\star\mu,\pi\star \bar v}$ it has to be of the form 
		 $(v_1',v_1,v_3')$ with $v_1'$ and $v_3'$ pairwise distinct and 
		 $v_1'$ and $v_3'$ different from $v_1$. We thus see that $\tilde{P}_{\pi\star\mu,\pi\star \bar v}=\tilde{P}_{\mu,\bar v}$ for $\tilde{P}_{\mu,\bar v}$ from the previous example. Suppose that
equality~(\ref{eq:tildeP}) does not hold for $\pi\star\mu$ and triples $\pi\star \bar v$ and $\pi\star\bar w$. Since we assume that $\chi_{G,3}^{(t)}(\bar v)=\chi_{H,3}^{(t)}(\bar w)$,
Observation~\ref{obs:permute} implies that $\chi_{G,3}^{(t)}(\pi\star\bar v)=\chi_{H,3}^{(t)}(\pi\star\bar w)$. Let us assume that $\pi\star\bar v$ and $\pi\star\bar w$ are assigned colour $c''$ by $\wl{3}$ in round $t$. Furthermore, we suppose again that $\tilde{P}_{\pi\star\mu,\pi\star \bar v}$ has more than $m$ triples of colour $c'$, assigned by $\wl{3}$ in round $t'$, whereas $\tilde{Q}_{\pi\star\mu,\pi\star \bar w}$ has less than $m$ such triples. We can now use the formula $\varphi(x_1,x_2,x_3)$ defined as
$$
	\psi_{c''}^{(t)}(x_1,x_2,x_3)\land \Bigl(
	\exists^{\geq m}(x_1',x_3')\, \psi_{c'}^{(t')}(x_1',x_2,x_3')\land
	  x_3\neq x_1' \land x_1'\neq x_2 \land x_3\neq x_2\Bigr)
$$
to distinguish $\pi\star\bar v$ from $\pi\star\bar w$. Indeed, by moving to the permuted versions, we can simply use the variable $x_2$ to ensure that triples $\bar v'$ have $v_1$ as second entry, as this is  now the second entry in $\pi\star\bar v=(v_2,v_1,v_3)$.
As a consequence, $G\models\varphi[\pi\star\bar v]$ but $H\not\models\varphi[\pi\star\bar w]$.
Then, similarly as in the proof of Observation~\ref{obs:permute}, we obtain that 
$G\models\pi\star\varphi[\bar v]$ and $H\not\models\pi\star\varphi[\bar w]$, contradicting $\chi_{G,3}^{(t)}(\bar v)=\chi_{H,3}^{(t)}(\bar w)$ as well. \hfill\qed
\end{example}

To carry out the strategy as outlined in the example, we need to find a good permutation
$\pi$ of $[k]$, show that $\tilde{P}_{\mu,\bar v}=\tilde{P}_{\pi\star \mu,\pi\star v}$ (and
thus also $\tilde{Q}_{\mu,\bar w}=\tilde{Q}_{\pi\star \mu,\pi\star w}$), and finally, construct a formula in $\countinglogic{k}$ of quantifier rank at most $t$ that allows us to distinguish $\bar v$ from $\bar w$.

We start by defining when a permutation is good in terms of a property of equality patterns.
More specifically, we say that an equality pattern $\mu\in[n]^{2k}/_\sim$ is ``good'' if it satisfies the following condition, expressed in terms of the partition $[2k]=I_1\uplus\cdots\uplus I_r$ of $\mu$:
\begin{equation*}
	\text{For every used constant  class $I_s$: If $i$ is the smallest index satisfying $k+i\in I_s$, then  $i\in I_s$} \tag{e}
\end{equation*}
Intuitively, this condition corresponds to the requirement that when $\bar v'\in \tilde{P}_{\mu,\bar v}$ and $v_j'=v_{\mathsf{rep}(I_s)}$ for a used constant class $I_s$,
then if $i$ is the smallest such index, i.e., $k+i\in I_s$ and thus $v_i'=v_{\mathsf{rep}(I_s)}$, then
$v_i'$ (and thus also all $v_j'$ with $k+j\in I_s$) can be assumed to be equal to $v_i$, where $v_{\mathsf{rep}(I_s)}=v_i$.

We next show that we can assume that condition (e) holds by replacing $\mu$  by a permutation $\pi\star\mu$ thereof and furthermore, $\tilde{P}_{\mu,\bar v}=\tilde{P}_{\pi\star\mu,\pi\star\bar v}$.

\begin{observation}\label{obs:good}
For every $\mu\in[n]^{2k}/_\sim$ and $\bar v$, there exists a permutation $\pi$ of $[k]$
such that
$\tilde{P}_{\mu,\bar v}=\tilde{P}_{\pi\star\mu,\pi\star\bar v}$ and $\pi\star\mu$ is good. \end{observation}
\begin{proof}
Let us represent $\mu$ by $[2k]=I_1\uplus\cdots\uplus I_r$. 
Consider the following permutation $\pi$ of $[k]$: For each used constant  class $I_s$, we first
define $\pi(i):=j$ with $j$ such that $k+j$ the smallest index in $I_s$ and such that $\mathsf{rep}(I_s)=i$. We then extend $\pi$ to
a permutation of $[k]$ in an arbitrary way.

Let us first show that $\pi\star\mu$ is good, i.e.,  that condition (e) is satisfied. Take a used constant class $\pi\star I_s$ in $\pi\star\mu$ and let $j$ be the smallest index such that $k+j\in \pi\star I_s$. By definition of $\pi\star I_s$, $k+j$ is also the smallest index in $I_s$ larger than $k$. As a consequence, for $i=\mathsf{rep}(I_s)$, $\pi(i)$ is mapped to $j$ by definition of $\pi$. We note that $\pi(i)=j\in \pi\star I_s$, as desired.

Furthermore, to verify $\tilde{P}_{\mu,\bar v}=\tilde{P}_{\pi\star\mu,\pi\star\bar v}$ it suffices to observe that $k+i\in I_s$ if and only if $k+i\in \pi\star I_s$. In other words, classes in $\pi\star \mu$ and $\mu$ agree on indexes larger than $k$. This implies that $k$-tuples in $\tilde{P}_{\mu,\bar v}$ and $\tilde{P}_{\pi\star\mu,\pi\star\bar v}$ satisfy the same conditions (a) and (b). It remains to verify that they also satisfy the same conditions (c). That is, consider a used constant  class $I_s$ and $k+i\in I_s$.
For $\bar v'$ to be in $\tilde{P}_{\mu,\bar v}$, $v'_i=v_{\mathsf{rep}(I_s)}$. 
Similarly, for $\bar v'$ to be in $\tilde{P}_{\pi\star\mu,\pi\star\bar v}$,  $v'_i=(\pi\star v)_{\mathsf{rep}(\pi\star I_s)}$. We show that $v_{\mathsf{rep}(I_s)}=(\pi\star v)_{\mathsf{rep}(\pi\star I_s)}$. Indeed, we observe that $(\pi\star \bar v)_{\mathsf{rep}(\pi\star I_s)}$ is equal to $\bar v_{\pi^{-1}(\mathsf{rep}(\pi\star I_s))}$. Let $j=\mathsf{rep}(\pi\star I_s)$, i.e., $j$ is the smallest index of the form $\pi(j')$ for $j'\in I_s$ with $j'\leq k$. Hence, $\bar v_{\pi^{-1}(\mathsf{rep}(\pi\star I_s))}=v_{j'}$ for some $j'\in I_s$ with $j'\leq k$. As a consequence, $v_i'=v_{j'}=v_{\mathsf{rep}(I_s)}$
since $j'$ and $\mathsf{rep}(I_s)$ both belong to $I_s$. 
\end{proof}

We are now finally ready to conclude the proof of the Key Lemma. Consider $\bar v\in (V(G))^k$ and $\bar w\in (V(H))^k$ satisfying $\chi_{G,k}^{(t)}(\bar v)=\chi_{H,k}^{(t)}(\bar w)$. We have seen earlier, in Observation~\ref{obs:decomp}, that to ensure that equality~($\ddagger\ddagger$) holds, it suffices to verify that equation~(\ref{eq:tildeP}) holds. Furthermore, Observation~\ref{obs:good} tells us that we can find a permutation $\pi$ such that $\pi\star\mu$ is good, and that it suffices to verify that $\chi_{G,k}^{(t)}(\bar v)=\chi_{H,k}^{(t)}(\bar w)$ implies
$$
\Bldbl \chi_{G,k}^{(t')}(\bar v')\bigm| \bar v'\in \tilde{P}_{\pi\star\mu,\pi\star\bar v} \Brdbl=
\Bldbl \chi_{H,k}^{(t')}(\bar w')\bigm| \bar w'\in \tilde{Q}_{\pi\star\mu,\pi\star\bar w} \Brdbl.
$$
Given Observation~\ref{obs:permute}, we can equivalently assume  $\chi_{G,k}^{(t)}(\pi\star\bar v)=\chi_{H,k}^{(t)}(\pi\star\bar w)$ instead of
 $\chi_{G,k}^{(t)}(\bar v)=\chi_{H,k}^{(t)}(\bar w)$.

All combined, it remains to show the following observation. Here, we restrict ourselves to equality patterns that have used constant classes. Equality patterns with only unused constant classes are dealt with afterwards.
\begin{observation} Let $\bar v\in (V(G))^k$ and $\bar w\in(V(H))^k$ satisfying $\chi_{G,k}^{(t)}(\bar v)=\chi_{H,k}^{(t)}(\bar w)$. Let $\mu\in[n]^{2k}/_\sim$ be a good equality pattern with at least one used constant class. Then,
\begin{equation}\Bldbl \chi_{G,k}^{(t')}(\bar v')\bigm| \bar v'\in \tilde{P}_{\mu,\bar v} \Brdbl=
\Bldbl \chi_{H,k}^{(t')}(\bar w')\bigm| \bar w'\in \tilde{Q}_{\mu,\bar w} \Brdbl.  \label{eq:tildePgood}
\end{equation}
\end{observation}
\begin{proof}
Suppose, for the sake of contradiction, that~(\ref{eq:tildePgood}) does not hold.
We assume that $\bar v$ and $\bar w$ are assigned colour $c$ by $\wl{k}$ in round $t$. For the equality~(\ref{eq:tildePgood}) not to hold, we assume  that there are more than $m$ $k$-tuples in $\tilde{P}_{\mu,\bar v}$
of colour $c'$, assigned by $\wl{k}$ in round $t'$, but $\tilde{Q}_{\mu,\bar w}$ has less than $m$ such $k$-tuples. We will express this property by means of a $\countinglogic{k}$ formula of quantifier rank at most $t$.
Let $\mathsf{cidx}$ be the set of indexes $i$ such that $k+i$ is the smallest index (larger than $k$) in a used constant  class $I_s$ of $\mu$. By our assumption that there is at least one used constant class for $\mu$, $\mathsf{cidx}$ is non-empty. 
We denote by $\mathsf{class}(i)$ the used constant  class associated with $i$. We remark that $\mathsf{class}(i)\neq \mathsf{class}(j)$ for $i,j\in\mathsf{cidx}$ and $i\neq j$. Indeed, otherwise $I_s$ contains two smallest distinct entries $k+i$ and $k+j$.
Let $\mathsf{vidx}=\{1,\ldots,k\}\setminus \mathsf{cidx}$. We remark that when $k+i\in I_s$ for a variable class $I_s$, then $i\in\mathsf{vidx}$. Similarly, when $k+i\in I_s$ is a used constant class and $k+i$ is not the smallest such entry, $i\in\mathsf{vidx}$.

Consider now the formula $\varphi(x_1,\ldots,x_k)$ defined as
\begin{multline*}\psi_c^{(t)}(x_1,\ldots,x_k)\land \biggl(\exists^{\geq m}(x_i\mid i\in\mathsf{vidx})\, \Bigl( \psi_{c'}^{(t')}(x_1,\ldots,x_k) \land 
\bigwedge_{\substack{I_s\\\text{variable}}}\!\!\bigwedge_{\substack{\phantom{(i)}k+i,k+j\in I_s\\\phantom{j\neq i}}} x_i'=x_j'\land {}\\	
\bigwedge_{i\in\mathsf{cidx}}\bigwedge_{\substack{k+j\in \mathsf{class}(i)\\j\neq i}}
 x_j'=x_i \land 
 \bigwedge_{\substack{I_s,I_{s'},s\neq s'\\\text{variable}}}\bigwedge_{\substack{k+i\in I_s\\k+j\in I_{s'}}} x_i'\neq x_j'   \land{}\\
 \bigwedge_{i\in\mathsf{cidx}}
   \bigwedge_{\substack{I_s\\ \text{variable}}}\bigwedge_{\substack{k+j\in I_s\\\phantom{k+j\in I_s}}} x_j'\neq x_i
\Bigr)\biggr).
\end{multline*}
Before showing that this formula indeed expresses what we want, we observe that its quantifier rank is at most
$\mathsf{max}\{t,t'+|\mathsf{vidx}|\}$. Indeed, recall from Section~\ref{sec:background} that the sub-formula, using the quantifier $\exists^{\geq m}(x_i\mid i\in\mathsf{vidx})$, is equivalent to a formula in $\countinglogic{k}$ of quantifier at most $t'+|\mathsf{vidx}|$. Since there is at least one used constant class in $\mu$ and $|\mathsf{vidx}|\leq k-1$ and
thus $t'+|\mathsf{vidx}|\leq t'+ k -1= t$, as desired. We further observe that this is a formula only using variables $x_1,\ldots,x_k$, and hence it is in $\countinglogic{k}$.\looseness=-1

We next show that  $G\models\varphi[\bar v]$ whereas $H\not\models\varphi[\bar w]$, contradicting $\chi_{G,k}^{(t)}(\bar v)=\chi_{H,k}^{(t)}(\bar w)$.
To verify $G\models\varphi[\bar v]$ we first observe that $G\models\psi_{c}^{(t)}[\bar v]$ because $\chi_{G,k}^{(t)}(\bar v)=c$. Conversely, $G\models\varphi[\bar v]$ necessarily implies that $G\models\psi_{c}^{(t)}[\bar v]$ and thus $\chi_{G,k}^{(t)}(\bar v)=c$.

For the sub-formula under the quantifier $\exists^{\geq m}(x_i\mid i\in\mathsf{vidx})$,
let $\alpha:\{x_1,\ldots,x_k\}\to V(G)$ be the assignment corresponding to $\bar v$, i.e., $\alpha(x_i)=v_i$. Let $\ell:=|\mathsf{vidx}|$.
If $G\models\varphi[\bar v]$ then this implies that there are more than $m$ $\ell$-tuples $(v_i'\mid i\in\mathsf{vidx})$ in $(V(G))^\ell$ such that 
\begin{multline*}G\models
\psi_{c'}^{(t')}[\alpha(x_i/v_i'\mid i\in\mathsf{vidx})]
 \land \underbrace{\bigwedge_{\substack{I_s\\\text{variable}}}\!\!\bigwedge_{\substack{\phantom{(i)}k+i,k+j\in I_s\\\phantom{j\neq i}}} v_i'=v_j'}_{\text{(i)}}\land 
\underbrace{\bigwedge_{i\in\mathsf{cidx}}\bigwedge_{\substack{k+j\in \mathsf{class}(i)\\j\neq i}}
 v_j'=v_i}_{\text{(ii)}} \land {}
  \\
{}  \underbrace{\bigwedge_{\substack{I_s,I_{s'},s\neq s'\\\text{variable}}}\bigwedge_{\substack{k+i\in I_s\\k+j\in I_{s'}}} v_i'\neq v_j'}_{\text{(iii)}}
  \land
\underbrace{\bigwedge_{i\in\mathsf{cidx}}
  \bigwedge_{\substack{I_s\\ \text{variable}}}\bigwedge_{\substack{k+j\in I_s\\\phantom{k+j\in I_s}}} v_j'\neq v_i}_{\text{(iv)}}
\end{multline*}
holds.
We verify that for each $(v_i'\mid i\in\mathsf{vidx})$ defined above,
the tuple $\bar v'':=\alpha(x_i/v_i'\mid i\in\mathsf{vidx})$
 is a tuple in $\tilde{P}_{\mu,\bar v}$ (here, we identify an assignment with its image).
We verify that conditions (a), (b) and (c) are satisfied for $\mu$. For condition (a),
take $k+i$ and $k+j$ in a variable class $I_s$. We observed before that for such $i$ and $j$, $i,j\in\mathsf{vidx}$ and thus $v_i''=v_i'$ and $v_j''=v_j'$. Hence, the equality conditions $v_i'=v'_j$ in the sub-formula (i) ensure that condition (a) is satisfied for variable classes. Next, take $k+j$  in 
a used constant class $I_s$. Suppose that $\mathsf{class}(i)=I_s$ and thus $i\in\mathsf{cidx}$. To satisfy conditions (a) and (c), we need $v_j''=v_{\mathsf{rep}(I_s)}$. We now observe that $v_i=v_{\mathsf{rep}(I_s)}$ and $v_{j}''=v_j'$ for $j\neq i$. Hence the equalities $v_j'=v_i$ with $k+j\in I_s$ and $j\neq i$ in the sub-formula (ii) ensure that conditions (a) and (c) are satisfied for used constant classes.
Finally, for condition (b) we argue in a similar way. More specifically, consider two distinct variable
classes $I_s$ and $I_{s'}$, and let $k+i\in I_s$ and $k+j\in I_{s'}$. For $\bar v''$ to satisfy condition (b), $v_i''\neq v_j''$. Since $i$ and $j$ are in $\mathsf{vidx}$, the equalities $v_i'=v_j'$ in the sub-formula (iii) ensure that condition (b) is satisfied for
distinct variables classes. Similarly, let $I_s$ be a variable class and $I_{s'}$ are used constant class. Assume that $I_{s'}=\mathsf{class}(i)$. We know from sub-formula (ii) that for all $k+j\in I_{s'}$, $j\neq i$,  $v_j''=v_j'=v_i$. To satisfy condition (b), we need
$v_j''=v_j'$ for  $k+j\in I_s$ to be distinct from any $v_{j''}''$ for  $k+j''\in I_{s'}$. This is ensured by the inequalities $v_j'\neq v_i$ in the sub-formula (iv) since we have $v_{j''}''=v_i$ for all $k+j''\in I_{s'}$. Finally, let $I_s$ and $I_{s'}$ be two distinct used constant classes. Assume that $I_{s}=\mathsf{class}(i)$ and $I_{s'}=\mathsf{class}(j)$. Then the equalities in sub-formula (ii) ensure that for all
$k+i'\in I_{s}$, $i'\neq i$ and $k+j'\in I_{s'}$, $j'\neq j$, $v_{i'}''=v_{i'}'=v_i$ and
$v_{j'}''=v_{j'}'=v_j$. It now suffices to observe that $v_i\neq v_j$ since $i$ and $j$ belong to different used constant classes. Hence. $v_{i'}''\neq v_{j'}''$ as desired by condition (b). As a consequence, $\bar v''\in \tilde{P}_{\mu,\bar v}$. Clearly,
since $G\models\psi_{c'}^{(t')}[\bar v'']$, $\bar v''$ has colour $c'$ assigned by $\wl{k}$ in round $t'$. We may thus conclude that when $G\models\varphi[\bar v]$ that there are more that $m$ $k$-tuples in $\tilde{P}_{\mu,\bar v}$ of colour $c'$, assigned by $\wl{k}$ in round $t$. Conversely, suppose that are more than $m$ such tuples in $\tilde{P}_{\mu,\bar v}$. Then clearly, $G\models\varphi[\bar v]$. The same holds for $H$ and $\bar w$. By assumption, $G\models\varphi[\bar w]$ but $H\not\models\varphi[\bar w]$, contradicting   
$\chi_{G,k}^{(t)}(\bar v)=\chi_{H,k}^{(t)}(\bar w)$. In other words, the equality~(\ref{eq:tildePgood}) must hold.
\end{proof}

In the previous observation we assumed that $\mu$ has at least one used constant class. Indeed, otherwise, we need to existentially quantify over $k$ variables in the constructed formula $\varphi$. We note that when no used constant classes exist, this implies that $\bar v'\in \tilde{P}_{\mu,\bar v}$ if and only if conditions (a) and (b) are satisfied for variables classes. In the following, we assume that $\mu$ has no used constant classes.
\begin{observation}
Let $\mu\in[n]^{2k}/_\sim$ be an equality pattern without used constant classes. 
If  $G\equiv_{\wl{k}}^{t'} H$, then
\begin{equation}
	\Bldbl \chi_{G,k}^{(t')}(\bar v')\bigm| \bar v'\in\tilde{P}_{\mu,\bar v} \Brdbl=
\Bldbl \chi_{H,k}^{(t')}(\bar w')\bigm| \bar w'\in \tilde{Q}_{\mu,\bar w} \Brdbl. \label{eq:finalh}
\end{equation}
for any $\bar v\in (V(G))^k$ and $\bar w\in (V(H))^k$.
\end{observation}
\begin{proof}
As mentioned above, for $\bar v'$ to be in $\tilde{P}_{\mu,\bar v}$ it simply needs to satisfy $v'_i=v_j'$ whenever
$k+i,k+j\in I_s$ with $I_s$ a variable class, and $v_i'\neq v_j'$ whenever $k+i\in I_s$, $k+j\in I_{s'}$ with $s\neq s'$ and $I_s$ and $I_{s'}$ variables classes. In other words,
due the absence of used constant classes, there is no relationship between $\bar v$ and $\bar v'$. This implies that we replace $\bar v'\in\tilde{P}_{\mu,\bar v}$ by $\bar v'\in\tau$ with $\tau\in[n]^k/_\sim$ represented by $[k]=I_1\uplus\cdots\uplus I_{r'}$ with $I_s:=\{k-i\mid i\in I_s\}$
and $I_s$ a variable class in $\mu$. As a consequence, instead of verifying the equality~(\ref{eq:finalh}) it suffices to verify
$$	\Bldbl \chi_{G,k}^{(t')}(\bar v')\bigm| \bar v'\in\tau\Brdbl=
\Bldbl \chi_{H,k}^{(t')}(\bar w')\bigm| \bar w'\in \tau\Brdbl.
$$
We have observed before, however, that 
$\chi_{G,k}^{(t')}(\bar v')=\chi_{H,k}^{(t')}(\bar w')$ implies that $\bar v'\sim\bar w'$ and thus both $\bar v$ and $\bar w$ belong to $\tau$.
Given that $G\equiv_{\wl{k}}^{t'} H$, or in order words,
\begin{equation}
	\Bldbl \chi_{G,k}^{(t')}(\bar v')\bigm| \bar v'\in (V(G))^k \Brdbl=
\Bldbl \chi_{H,k}^{(t')}(\bar w')\bigm| \bar w'\in (V(H))^k\Brdbl, \label{eq:finalhtau}
\end{equation}
we can indeed infer that the equality~(\ref{eq:finalhtau}) holds, as desired.
 \end{proof}
This concludes the proof of the Key Lemma.\hfill\qed

\section{Conclusion}\label{sec:conclude}
We have shown that $\igns{k}$ are equally expressive as $\wl{k}$ in distinguishing graphs, hereby answering a question raised by \citet{openprob}. As part of the proof, we observe that a single layer of a $\ign{k}$ corresponds to $k-1$ iterations of $\wl{k}$. This may result in $\igns{k}$ to quicker distinguish graphs than $\igns{k}$. The analysis of $\igns{k}$ in terms of equality patterns hints towards equally powerful but less computationally intensive variants of $\igns{k}$ in which certain equality patterns are disallowed. In this way, one can envisage $\igns{k}$ parameterised by a set of allowed equality patterns. In this way, one can obtain $\wl{k}$ and $\igns{k}$ as special cases, and tweak the correspondence between iterations of $\wl{k}$ and layers of $\igns{k}$ as one seems fit.


\begin{thebibliography}{28}
\providecommand{\natexlab}[1]{#1}
\providecommand{\url}[1]{\texttt{#1}}
\expandafter\ifx\csname urlstyle\endcsname\relax
  \providecommand{\doi}[1]{doi: #1}\else
  \providecommand{\doi}{doi: \begingroup \urlstyle{rm}\Url}\fi

\bibitem[Arvind et~al.(2020)Arvind, Fuhlbr{\"{u}}ck, K{\"{o}}bler, and
  Verbitsky]{ARVIND202042}
V.~Arvind, Frank Fuhlbr{\"{u}}ck, Johannes K{\"{o}}bler, and Oleg Verbitsky.
\newblock On {W}eisfeiler-{L}eman invariance: Subgraph counts and related graph
  properties.
\newblock \emph{Journal of Computer and System Sciences}, 113:\penalty0 42 --
  59, 2020.
\newblock URL \url{https://doi.org/10.1016/j.jcss.2020.04.003}.

\bibitem[Azizian \& Lelarge(2020)Azizian and
  Lelarge]{azizian2020characterizing}
Waïss Azizian and Marc Lelarge.
\newblock Characterizing the expressive power of invariant and equivariant
  graph neural networks.
\newblock \emph{CoRR}, abs/2006.15646, 2020.
\newblock URL \url{https://arxiv.org/abs/2006.15646}.

\bibitem[Barcel{\'o} et~al.(2020)Barcel{\'o}, Kostylev, Monet, P{\'e}rez,
  Reutter, and Silva]{barcelo2019logical}
Pablo Barcel{\'o}, Egor~V Kostylev, Mikael Monet, Jorge P{\'e}rez, Juan
  Reutter, and Juan~Pablo Silva.
\newblock The logical expressiveness of graph neural networks.
\newblock In \emph{International Conference on Learning Representations
  {(ICLR)}}, 2020.
\newblock URL \url{https://openreview.net/forum?id=r1lZ7AEKvB}.

\bibitem[Cai et~al.(1992)Cai, F{\"{u}}rer, and Immerman]{CaiFI92}
Jin{-}{Y}i Cai, Martin F{\"{u}}rer, and Neil Immerman.
\newblock An optimal lower bound on the number of variables for graph
  identifications.
\newblock \emph{Combinatorica}, 12\penalty0 (4):\penalty0 389--410, 1992.
\newblock URL \url{https://doi.org/10.1007/BF01305232}.

\bibitem[Chen et~al.(2020)Chen, Chen, Villar, and Bruna]{chen2020graph}
Zhengdao Chen, Lei Chen, Soledad Villar, and Joan Bruna.
\newblock Can graph neural networks count substructures?
\newblock \emph{arXiv}, 2020.
\newblock URL \url{https://arxiv.org/abs/2002.04025}.

\bibitem[Dell et~al.(2018)Dell, Grohe, and Rattan]{DellGR18}
Holger Dell, Martin Grohe, and Gaurav Rattan.
\newblock Lov{\'{a}}sz meets {W}eisfeiler and {L}eman.
\newblock In \emph{Proceedings of the 45th International Colloquium on
  Automata, Languages, and Programming, {(ICALP)}}, volume 107 of
  \emph{LIPIcs}, pp.\  40:1--40:14. Schloss Dagstuhl - Leibniz-Zentrum
  f{\"{u}}r Informatik, 2018.
\newblock URL \url{https://doi.org/10.4230/LIPIcs.ICALP.2018.40}.

\bibitem[F{\"{u}}rer(2017)]{Furer17}
Martin F{\"{u}}rer.
\newblock On the combinatorial power of the {W}eisfeiler-{L}ehman algorithm.
\newblock In \emph{Proceedings of the 10th International Conference on
  Algorithms and Complexity {(CIAC)}}, volume 10236 of \emph{Lecture Notes in
  Computer Science}, pp.\  260--271, 2017.
\newblock URL \url{https://doi.org/10.1007/978-3-319-57586-5\_22}.

\bibitem[Geerts(2019)]{Geerts19}
Floris Geerts.
\newblock On the expressive power of linear algebra on graphs.
\newblock In \emph{Proceedings of the 22nd International Conference on Database
  Theory {(ICDT)}}, volume 127 of \emph{LIPIcs}, pp.\  7:1--7:19. Schloss
  Dagstuhl - Leibniz-Zentrum f{\"{u}}r Informatik, 2019.
\newblock URL \url{https://doi.org/10.4230/LIPIcs.ICDT.2019.7}.

\bibitem[Geerts(2020)]{geerts2020walk}
Floris Geerts.
\newblock Walk message passing neural networks and second-order graph neural
  networks.
\newblock \emph{ArXiv}, 2020.
\newblock URL \url{https://arxiv.org/abs/2006.09499}.

\bibitem[Geerts et~al.(2020)Geerts, Mazowiecki, and P{\'e}rez]{geerts2020let}
Floris Geerts, Filip Mazowiecki, and Guillermo~A P{\'e}rez.
\newblock Let's agree to degree: Comparing graph convolutional networks in the
  message-passing framework.
\newblock \emph{arXiv}, 2020.
\newblock URL \url{https://arxiv.org/abs/2004.02593}.

\bibitem[Gilmer et~al.(2017)Gilmer, Schoenholz, Riley, Vinyals, and
  Dahl]{GilmerSRVD17}
Justin Gilmer, Samuel~S. Schoenholz, Patrick~F. Riley, Oriol Vinyals, and
  George~E. Dahl.
\newblock Neural message passing for quantum chemistry.
\newblock In \emph{Proceedings of the 34th International Conference on Machine
  Learning {(ICML)}}, volume~70, pp.\  1263--1272, 2017.
\newblock URL \url{{http://proceedings.mlr.press/v70/gilmer17a/gilmer17a.pdf}}.

\bibitem[Grohe(2017)]{grohe_2017}
Martin Grohe.
\newblock \emph{Descriptive Complexity, Canonisation, and Definable Graph
  Structure Theory}.
\newblock Lecture Notes in Logic. Cambridge University Press, 2017.
\newblock URL \url{https://doi.org/10.1017/9781139028868}.

\bibitem[Grohe(2020)]{Grohe20}
Martin Grohe.
\newblock word2vec, node2vec, graph2vec, x2vec: Towards a theory of vector
  embeddings of structured data.
\newblock In \emph{Proceedings of the 39th {ACM} {SIGMOD-SIGACT-SIGAI}
  Symposium on Principles of Database Systems {(PODS)}}, pp.\  1--16. {ACM},
  2020.
\newblock URL \url{https://doi.org/10.1145/3375395.3387641}.

\bibitem[Grohe \& Otto(2015)Grohe and Otto]{grohe_otto_2015}
Martin Grohe and Martin Otto.
\newblock Pebble games and linear equations.
\newblock \emph{The Journal of Symbolic Logic}, 80\penalty0 (3):\penalty0
  797–844, 2015.
\newblock URL \url{https://doi.org/10.1017/jsl.2015.28}.

\bibitem[Kipf \& Welling(2017)Kipf and Welling]{kipf-loose}
Thomas~N. Kipf and Max Welling.
\newblock Semi-supervised classification with graph convolutional networks.
\newblock In \emph{International Conference on Learning Representations
  {(ICLR)}}, 2017.
\newblock URL \url{https://openreview.net/forum?id=SJU4ayYgl}.

\bibitem[Kondor et~al.(2018)Kondor, Son, Pan, Anderson, and
  Trivedi]{kondor2018covariant}
Risi Kondor, Hy~Truong Son, Horace Pan, Brandon Anderson, and Shubhendu
  Trivedi.
\newblock Covariant compositional networks for learning graphs.
\newblock In \emph{International Conference on Learning Representations
  {(ICLR)}}, 2018.
\newblock URL \url{https://openreview.net/forum?id=S1TgE7WR-}.

\bibitem[Lichter et~al.(2019)Lichter, Ponomarenko, and Schweitzer]{LichterPS19}
Moritz Lichter, Ilia Ponomarenko, and Pascal Schweitzer.
\newblock Walk refinement, walk logic, and the iteration number of the
  {W}eisfeiler-{L}eman algorithm.
\newblock In \emph{Proceedings of the 34th Annual {ACM/IEEE} Symposium on Logic
  in Computer Science {(LICS)}}, pp.\  1--13, 2019.
\newblock URL \url{https://doi.org/10.1109/LICS.2019.8785694}.

\bibitem[Loukas(2020)]{Loukas2020What}
Andreas Loukas.
\newblock What graph neural networks cannot learn: depth vs width.
\newblock In \emph{International Conference on Learning Representations
  {(ICLR)}}, 2020.
\newblock URL \url{https://openreview.net/forum?id=B1l2bp4YwS}.

\bibitem[Maron et~al.(2019{\natexlab{a}})Maron, Ben-Hamu, and Lipman]{openprob}
Haggai Maron, Heli Ben-Hamu, and Yaron Lipman.
\newblock Open problems: Approximation power of invariant graph networks.
\newblock In \emph{NeurIPS 2019 Graph Representation Learning Workshop},
  2019{\natexlab{a}}.
\newblock URL \url{https://grlearning.github.io/papers/31.pdf}.

\bibitem[Maron et~al.(2019{\natexlab{b}})Maron, Ben{-}Hamu, Serviansky, and
  Lipman]{DBLP:conf/nips/MaronBSL19}
Haggai Maron, Heli Ben{-}Hamu, Hadar Serviansky, and Yaron Lipman.
\newblock Provably powerful graph networks.
\newblock In \emph{Advances in Neural Information Processing Systems 32: Annual
  Conference on Neural Information Processing Systems {(NeurIPS)}}, pp.\
  2153--2164, 2019{\natexlab{b}}.
\newblock URL
  \url{http://papers.nips.cc/paper/8488-provably-powerful-graph-networks}.

\bibitem[Maron et~al.(2019{\natexlab{c}})Maron, Ben-Hamu, Shamir, and
  Lipman]{maron2018invariant}
Haggai Maron, Heli Ben-Hamu, Nadav Shamir, and Yaron Lipman.
\newblock Invariant and equivariant graph networks.
\newblock In \emph{International Conference on Learning Representations
  {(ICLR)}}, 2019{\natexlab{c}}.
\newblock URL \url{https://openreview.net/forum?id=Syx72jC9tm}.

\bibitem[Morris et~al.(2019)Morris, Ritzert, Fey, Hamilton, Lenssen, Rattan,
  and Grohe]{grohewl}
Christopher Morris, Martin Ritzert, Matthias Fey, William~L. Hamilton, Jan~Eric
  Lenssen, Gaurav Rattan, and Martin Grohe.
\newblock Weisfeiler and {L}eman go neural: Higher-order graph neural networks.
\newblock In \emph{Proceedings of the 33rd {AAAI} Conference on Artificial
  Intelligence {(AAAI)}}, pp.\  4602--4609, 2019.
\newblock URL \url{https://doi.org/10.1609/aaai.v33i01.33014602}.

\bibitem[NT \& Maehara(2020)NT and Maehara]{nt2020graph}
Hoang NT and Takanori Maehara.
\newblock Graph homomorphism convolution.
\newblock \emph{arXiv}, 2020.
\newblock URL \url{https://arxiv.org/abs/2005.01214}.

\bibitem[Sato(2020)]{Sato2020ASO}
Ryoma Sato.
\newblock A survey on the expressive power of graph neural networks.
\newblock \emph{arXiv}, 2020.
\newblock URL \url{https://arxiv.org/abs/2003.04078}.

\bibitem[Sato et~al.(2019)Sato, Yamada, and Kashima]{sato2019approximation}
Ryoma Sato, Makoto Yamada, and Hisashi Kashima.
\newblock Approximation ratios of graph neural networks for combinatorial
  problems.
\newblock In \emph{Advances in Neural Information Processing Systems 32: Annual
  Conference on Neural Information Processing Systems {(NeurIPS)}}, pp.\
  4083--4092, 2019.

\bibitem[Sato et~al.(2020)Sato, Yamada, and Kashima]{sato2020random}
Ryoma Sato, Makoto Yamada, and Hisashi Kashima.
\newblock Random features strengthen graph neural networks.
\newblock \emph{arXiv}, 2020.
\newblock URL \url{https://arxiv.org/abs/2002.03155}.

\bibitem[Scarselli et~al.(2009)Scarselli, Gori, Tsoi, Hagenbuchner, and
  Monfardini]{scarselli2008graph}
Franco Scarselli, Marco Gori, Ah~Chung Tsoi, Markus Hagenbuchner, and Gabriele
  Monfardini.
\newblock The graph neural network model.
\newblock \emph{{IEEE} Trans. Neural Networks}, 20\penalty0 (1):\penalty0
  61--80, 2009.
\newblock URL \url{https://doi.org/10.1109/TNN.2008.2005605}.

\bibitem[Xu et~al.(2019)Xu, Hu, Leskovec, and Jegelka]{xhlj19}
Keyulu Xu, Weihua Hu, Jure Leskovec, and Stefanie Jegelka.
\newblock How powerful are graph neural networks?
\newblock In \emph{International Conference on Learning Representations
  {(ICLR)}}, 2019.
\newblock URL \url{https://openreview.net/forum?id=ryGs6iA5Km}.

\end{thebibliography}
\end{document}